\documentclass[onecolumn, 11pt]{amsart}
\usepackage{graphicx}
\usepackage{amsmath,amssymb,dsfont}
\usepackage{color}
\usepackage{flushend}
\usepackage{subfig}% subcaptions for subfigures
\usepackage{lettrine}%  dropped capital letter at beginning of paragraph
\usepackage{float} 
\usepackage{algorithmic}
\usepackage[vlined,linesnumbered,boxruled]{algorithm2e}
\usepackage{cite}
 \newtheorem{thm}{\textbf{Theorem}}[section]
 
 \newtheorem{lem}{\textbf{Lemma}}[section]
 
 \newtheorem{defn}{\textbf{Definition}}[section]

 \newtheorem{rem}{Remark}[section]
 \newcommand{\X}{\mathcal{X}}

 \newcommand*{\QEDB}{\hfill\ensuremath{\square}}
 
 \usepackage{bm}
\newcommand{\real}{\mathbb{R}}
\newcommand{\natno}{\mathbb{N}}
\newcommand{\card}[1]{\left|#1\right|}

\newcommand{\prob}{\text{Pr}}
\newcommand{\pfsa}{T}
\newcommand{\state}{q}
\newcommand{\stateSet}{Q}
\newcommand{\alphabet}{a}
\newcommand{\alphabetSet}{\mathcal{A}}
\newcommand{\symb}{a}

\newcommand{\trFn}{\delta}
\newcommand{\emProb}{\tilde{\pi}}

\newcommand{\emMat}{\bm{\tilde{\Pi}}}
\newcommand{\depth}{D}
\newcommand{\trProb}{\pi}

\newcommand{\trMat}{\bm{\Pi}}
\newcommand{\stProb}{p}
\newcommand{\stProbVec}{\bm{\stProb}}

\usepackage{multirow} %for table
\usepackage{booktabs} %for table
\usepackage{colortbl} %for color table

\usepackage{color}
\title {
{\LARGE \bf
Markov Modeling of Time-Series Data using Symbolic Analysis}%$^\bigstar$}
}

\author{Devesh K. Jha}
\address{MERL, Cambridge, MA}
\email{ devesh.dkj@gmail.com, jha@merl.com}

\begin{document}

%\begin{document}

\maketitle
\pagestyle{plain}

%\linenumbers           % Insert line numbers
\begin{abstract}
Markov models are often used to capture the temporal patterns of sequential data for statistical learning applications. While the Hidden Markov modeling-based learning mechanisms are well studied in literature, we analyze a symbolic-dynamics inspired approach. Under this umbrella, Markov modeling of time-series data consists of two major steps- discretization of continuous attributes followed by estimating the size of temporal memory of the discretized sequence. These two steps are critical for the accurate and concise representation of time-series data in the discrete space. Discretization governs the information content of the resultant discretized sequence. On the other hand, memory estimation of the symbolic sequence helps to extract the predictive patterns in the discretized data. Clearly, the effectiveness of signal representation as a discrete Markov process depends on both these steps. In this paper, we will review the different techniques for discretization and memory estimation for discrete stochastic processes. In particular, we will focus on the individual problems of discretization and order estimation for discrete stochastic process. We will present some results from literature on partitioning from dynamical systems theory and order estimation using concepts of information theory and statistical learning. The paper also presents some related problem formulations which will be useful for machine learning and statistical learning application using the symbolic framework of data analysis. We present some results of statistical analysis of a complex thermoacoustic instability phenomenon during lean-premixed combustion in jet-turbine engines using the proposed Markov modeling method.

%Then, we will present a Bayesian approach for Markov modeling of time series data using concepts from Kolmogorov Complexity theory and the principle of Minimum description length. We will present results of the proposed approach on various data sets for prognostics and health monitoring of combustion instability in gas turbine engines and fatigue damage in polycrystalline alloy structures. 
\end{abstract}

\section{Motivation and Introduction}
Markov models are used for statistical learning of sequential data when we are interested in temporal behavior of data and there is a requirement to relax the i.i.d. assumptions on measured data. Various different models have been proposed and analyzed in literature. Some of the common approaches are Autoregressive or AR models~\cite{BJRL15} and the Hidden Markov models~\cite{B06}. In contrast, in this paper we are interested in a non-linear symbolic analysis-based Markov modeling of time-series data which is not so well studied in literature.  Symbolic Time Series Analysis (STSA)~\cite{DFT03} is a non-linear technique for representing temporal patterns in sequential data. STSA borrows concepts from symbolic dynamics and chaos theory for non-linear analysis of irregular time series data  for systems which inherently do not seem stochastic. Symbolic time series analysis consists of two critical steps, \textit{discretization} where the continuous attributes of the sequential data are projected onto a symbolic space which is followed by identification of concise probabilistic patterns that help compress the discretized data. Under this umbrella, \textit{finite-memory} Markov models have been shown to be a reasonable \textit{finite-memory} approximation (or representation) of systems with fading memory (e.g., engineering systems that exhibit stable orbits or mixing)~\cite{R04, MR14}. Once the continuous data is discretized, the memory estimate for the discretized sequence is used to compress it as a finite-memory Markov process, which is represented by a state transition matrix. This helps find the causal, dynamical structure intrinsic to the symbolic process we are investigating, ideally to extract all the patterns in it that may have any predictive power.  The transition matrix could be estimated by frequency counting under the assumption of infinite data and a uniform prior for all the elements of the transition matrix. It is noted that in this paper we only consider discrete Markov processes with finite memory (or order). 

It is desirable from a machine learning perspective to establish the fundamental limits of information that could be learned from the underlying data when it is compressed as a Markov chain using the symbolic analysis framework. This problem is complex from the following perspectives.
\begin{enumerate}
\item The underlying model for data is unknown which makes data discretization difficult to analyze.
\item The order of the discrete stochastic process depends on the discretization parameters (i.e., number and locations of partitions) and this coupling is difficult to analyze in the absence of any underlying model. 
\end{enumerate}
These two reasons make analysis of this problem very difficult from a dynamical systems perspective. As the data model is unknown, there is no metric to characterize the properties of the discrete sequence w.r.t. the original data. Furthermore, the behavior of the discrete sequence is governed by the discretization process. As a result, characterizing the final Markov model w.r.t. the original data is extremely difficult due to this composite process. Dynamical systems theory provides some consistent characterization of partitions but they are in general very difficult to estimate numerically in the absence of a model. On the other hand, there is no information-theoretic or statistical formalism for the discretization process in open literature. However, in the absence of a model, use of information-theoretic and statistical measures for data discretization seems natural.

While the discretization step decides the information content of the symbolic sequence, memory estimation is critical for concise, yet precise, representation of the discretized sequence. A lot of techniques could be found in open literature for discretization as well as memory estimation of Markov processes. For machine learning applications, most of the methods tie a technique with an end objective which can then be used to find a solution that satisfies the end objective. However, most approaches consider the problems of data discretization and memory estimation as separate; there is no unifying theory for signal representation as a Markov model. Moreover, there are no results on the interplay between the complexity of the discrete dynamical system and the discretization process e.g., the cardinality of the discrete set, the discretization technique, etc.. In particular, there is no Bayesian inference technique that combines these steps together to estimate the model, together with its parameters, for representation of the time series signal. It is noted that the word \textit{symbolization} is often interchangeably used for \textit{discretization} throughout the paper. The word \textit{order} is often interchangeably used for \textit{depth}.

Symbolization is carried out via partitioning of the phase-space of the system
within which the system dynamics evolves. In general, there are two main lines of thought behind the symbolization process, one inspired by dynamical systems theory and the other inspired by machine learning objectives.  Some approaches inspired by the dynamical systems theory could be found in~\cite{HJK04, BK05, A98, KB03}. Several partitioning methods have been
proposed in literature such as maximum entropy partition~\cite{RR06}, symbolic aggregate approximation (SAX)~\cite{SAX07},
maximally-bijective partition~\cite{SSS13}. In these approaches, the discretization is not tied to any performance objective expected from the data (e.g., in maximum entropy partitioning, each partition contains equal number of data points and thus it presents an unbiased discrete representation of the system). This is somewhat different from machine learning-inspired techniques where an optimization over the parameters of partitions is performed for near-optimal results; however, these techniques are always susceptible to over-fitting of data and may lead to overly-capable algorithms.  Most of the machine learning inspired discretization techniques are mainly based on some end objectives like class separability or some unsupervised metrics like entropy minimization of the resulting discrete data~\cite{YW02, F04}. Good reviews of partitioning methods
for time-series symbolization can be found in~\cite{LHTD2002, GLSLH2012}. A information theory-based approach to select alphabet size was presented recently in~\cite{SCRPL15}. Some empirical results regarding the comparison and performance of different methods have been reported in~\cite{SMEJ01, DKS95, P95}. Even though a lot of techniques have been reported in open literature, there is no standard approach to it. This is for the reason that the effectiveness of a discretization technique depends on a lot of factors like the nature of the dynamical system, the topological properties of the phase space of the system, the similarity metrics etc.. Moreover, in statistical learning problems, the underlying model generating the data is unknown and thus, it is difficult to impose a consistency criterion based on system behavior. Consequently, finding a universal rule which can be used for a wide variety of systems with consistent performance is evasive. To the best of author's knowledge, there are no consistency results reported in open literature for discretization of signals for representation and learning.%Some empirical results about the 

The next step in the process is modeling of the symbol sequence which allows further compression of the data as a Markov model. Working in the symbolic domain, the task is to
identify concise probabilistic models that relate the past, present and the
future states of the system under consideration. For Markov modeling of the symbol sequence, this is achieved by first estimating the depth (or size of memory) for the discrete symbolic process and then, estimating the approximate the stochastic model from the observed sequence. Various approaches have been reported in open literature for order estimation of Markov chains. Most of these approaches are inspired by information theory and makes use of concepts like the Minimum description length (MDL)~\cite{BRY98, CT06}, complexity theory~\cite{C12} and/or some information criterion like Akaike's Information Criterion (AIC)~\cite{T75} or Bayesian Information Criterion (BIC)~\cite{K81, CT06}. Some important results regarding consistency of these techniques have also been reported for order estimation. In~\cite{CS00, C02}, the authors show that the minimum description length Markov estimator will converge almost surely to the correct order if the alphabet size is bounded a priori. In~\cite{MW07, MW071, MW08, MW13}, the authors present a more direct way of order estimation by making use of convergence rate of $k$-block distributions. Some other techniques techniques based on the use of some heuristic information theory-based criterion for order estimation could be found in~\cite{PK13, ZDG01, MGZ89}. For machine learning application, the estimation process follows {\it wrapper} approach~\cite{B06} where a
wrapper search algorithm with a certain stopping criteria calls the main modeling
module to build several temporal models with varying depths and the search is stopped
when the stopping criteria such as information gain or entropy rate show
marginal improvement with the added complexity~\cite{RRSM07}. This approach
in essence creates all models first and then chooses one model based on some
given metric and threshold. Making multiple passes through the
data {\it searching} for correct depth is computationally intensive
and might be infeasible for large data sets that are common today.

Recently, a new approach for depth estimation has been proposed based on the
spectral properties of the one-step transition matrix~\cite{Srivastav2014}. An
upper bound on the size of temporal memory has been derived that requires a
single-pass through the time-series data. In~\cite{JSMR15}, the authors presented a rigorous
comparative evaluation of spectral property-based approach with three popular techniques
-- (1) log-likelihood, (2) state-splitting using entropy-rate, and (3) signal
reconstruction error based approach. The latter three techniques fall under the
wrapper approach of depth estimation and therefore computationally. They were found to be in close agreement about the estimated depth. 

Once these two parameters are estimated, the data is represented by a stochastic matrix for the inferred Markov chain. The stochastic matrix could be estimated by a Bayesian approach where the prior and posterior of the individual elements of the matrix can be represented by Dirichlet and Multinomial distributions, respectively. Under the assumption of finite state-space of the Markov chain and an uniform prior, it can be shown that the posterior estimates asymptotically converge to the Maximum Likelihood estimate for the individual elements. Thus, through this sequential process of discretization and order estimation, we render a generative Markov model for the data. This representation for the signal can then be  used for various learning algorithms like modeling, classification, pattern matching etc.. 

Even though we see that a lot of work has been done on discretization and approximation of discrete sequences as finite-order Markov chains, there is no theory or approach that ties them together for signal representation. Furthermore, as these pieces have not been studied together there is no result on consistency of signal representation as a Markov model, where the underlying model generating the data is unknown. As a result, statistical modeling of data for STSA-based Markov modeling is still largely practiced in an ad-hoc fashion following a \textit{wrapper}-inspired technique where the solution space is searched exhaustively for solutions and thus, still remains hugely dependent on domain expertise. The challenges are both algorithmic and computational. The algorithmic challenge is to synthesize a consistent framework with guarantees on performance while the computational challenge is to reduce computations required to arrive at a desired solution mainly inspired by machine learning applications. 

At this point we would like to point out that the presented formalism for statistical learning is different from the standard Hidden Markov models in the following ways.
\begin{itemize}
\item The state-space of these models is inferred from the observed data while in HMM, the state-space is never observed and thus it is difficult to infer.
\item Under the assumption that the observed sequence is a finite-order Markov chain, the estimation of parameters is simplified. The sufficient statistics could be obtained by estimating the order of the Markov chain and the symbol emission probabilities conditioned on the memory words of the discrete sequence. 
\end{itemize}  
Thus, the current framework of STSA is simplistic and can be easily used for embedded applications for the simplicity of the inference algorithms. The details are provided in the subsequent sections.

 In this paper, we will present the state-of-the-art mathematical formalism of these concepts which may be possibly used for signal representation. We present a review and propose some mathematical problems which are required to understand the symbolic time-series analysis framework. In order to do so, we will review and discuss the concept of Markov partitions from dynamical systems literature and order estimation of Markov chains from probability and information theory. We will study some properties of Markov partitions and present an estimation technique for order estimation of a Markov chain that guarantees asymptotic convergence to the true order of the Markov chain. The closure of the paper will discuss some possible directions of research. In particular, we will discuss some possible ways in the underlying framework of STSA can be modeled from an information-theoretic and statistical perspective.
 
The paper is organized in eight  sections including the current section. Section~\ref{sec:background} briefly presents the underlying mathematical preliminaries and prepares the background for presentation of the subsequent material. Section~\ref{sec:probform} presents the statement of the problem along with a list of major assumptions. Section~\ref{sec:partitions} presents some results on a particular type of well-behaved partitions studied in dynamical systems literature. Section~\ref{sec:OrderEstimation} presents an estimation procedure for order of a discrete Markov chain with guarantees on convergence of the estimates. Section~\ref{sec:Parameters} presents some details on estimation of statistical parameters for the Markov models. In section~\ref{sec:results}, we present some details of application to a complex thermoacoustic phenomena in combustion using the presented ideas of Markov modeling. Finally the paper is summarized and concluded in Section~\ref{sec:Conclusions} along with recommendations for future research.
% In particular, we present the mathematical formalisms using ideas from complexity theory for signal compression. Furthermore, we introduce various different measures which can be used to study the performance of the Markov models created from a data-driven perspective. The presented ideas are tested on data sets from various complex physical processes like thermoacoustic instability in combustion process, failure of polycrystalline alloys due to fatigue, complex two-phase fluid flows etc.. The effectiveness of the proposed methods is illustrated through various measures defined for modeling accuracy from a data-driven perspective and shows the advantages over the existing algorithms for compression using STSA.

\section{Mathematical Preliminaries and Background}\label{sec:background}
In this section, we explain briefly the related ideas and concepts for symbolic time-series analysis-based Markov modeling of data. In particular, we will introduce the concepts of discretization, order (memory) and Markov modeling. Symbolic analysis of time-series data is a recent approach where continuous
sensor data are converted to symbol sequences via partitioning of the continuous
domain~\cite{SAX07, RRSM07, CRSK10}. The discrete symbolic dynamic process is then compressed as a Markov chain whose states are collection of words over a finite alphabet with finite length. In this section, we provide briefly the related concepts which are required to explain the subsequent material. We present some relevant definitions and concepts which are used to describe the subsequent materials. 
\begin{defn}[\textbf{Pre-Image}]
Given $\phi:M \rightarrow L$ (assume $\phi$ is surjective), the image of $x \in M$ is $\phi(x) \in L$. Then, the pre-image of $y \in L$ is given by the set $\phi^{-1}(y) = \{x\mid \phi(x)=y \}$.
\end{defn}
\begin{defn}[\textbf{Homeomorphism}] \label{def:homeomorphism}
A function $\phi:M \rightarrow N$ from one metric space to another is continuous if, whenever $x_n\rightarrow x$ in $M$, then $\phi(x_n)\rightarrow \phi(x)$ in $N$. If $\phi$ is continuous, one-to-one, onto and has a continuous inverse, then we call $\phi$ a homeomorphism.
\end{defn}

\begin{defn}[\textbf{Dynamical System}]\label{def:dynsys}
A dynamical system $(M,\phi)$ consists of a compact metric space $M$ together with a continuous map $\phi:M\rightarrow M$. If $\phi$ is homeomorphism we call $(M,\phi)$ to be an invertible dynamical system.
\end{defn}

\begin{defn}[\textbf{Irreducible Dynamical System}]
A dynamical system $(M,\phi)$ is said to be irreducible if for every pair of open sets $U,V$ there exists $m\geq 0$ such that $\phi^mU \cap V \neq 0$.
\end{defn}

\begin{defn}[\textbf{Bilaterally Transitive}]
A point $p \in M$ is called to be \textit{bilaterally transitive} if the forward orbit $\{ \phi^n p\mid n>0 \}$ and the backward orbit $\{ \phi^n p \mid n<0\}$ are both dense in $M$. 
\end{defn}

A homeomorphism $\phi$ is said to be expansive if there exists a real number $c>0$ such that $d(\phi^np,\phi^nq)<c$ for all $n \in \mathds Z$, then $p=q$.\\
\begin{defn}
For two dynamical systems $(M,\phi)$ and $(L,\psi)$ we call the second a factor of the first and the first an expansion of the second if there exists a map $\pi$ of $M$ into $L$, which we call a factor map such that
\begin{itemize}
\item $\psi \pi =\pi \phi$.
\item $\pi$ is continuous.
\item $\pi$ is onto.
\end{itemize}
Furthermore, we say that $\pi$ is a finite factor map or that it is bounded one-to-one if
\begin{itemize}
\item there is a bound to the number of pre-images.
\item every doubly transitive point has a unique pre-image.
\end{itemize}
\end{defn}
We are interested in symbolic representation of dynamical systems $(M,\phi)$. To understand the idea, consider the case where $\phi$ is invertible so that all the iterates of $\phi$, positive and negative, are used. Then, to describe the orbits $\{\phi^n(y):n \in \mathds Z\}$ of points $y \in M$, we can try to use an approximate description constructed in the following way. Divide $M$ into a finite number of pieces $E_0,E_1,\dots, E_{r-1}$ which cover the space $M$ and are mutually disjoint. Then, we can track the orbit of a point $y \in M$ by keeping a record of which of these pieces $\phi^n(y)$ lands in. This yields a corresponding point $x=\dots x_{-1}x_0x_1\dots \in \{0,1,\dots r-1\}^{\mathds Z}= X_{[r]}$ (full $r$-shift space) defined by
\begin{equation}
\phi^n(y) \in E_{x_n} \text{ for } n\in \mathds Z \nonumber
\end{equation}
Thus for every $y\in M$, we get a point $x$ in the full $r$-shift, and the definition shows that the image $\phi(y)$ corresponds to the shifted point $\sigma(x)$.\\
More formally, the partitions can be defined as follows.

\begin{defn}[\textbf{Partitions}] We call a family of sets $\mathcal{R}=\{R_1,R_2,\dots,R_{N-1}\}$ a topological partition for a compact metric space $M$ if the following holds true.
\begin{enumerate}
\item each $R_i$ is open;
\item $R_i\cap R_j =\emptyset$, $i\neq j$;
\item $M=\overline{R}_1\cup \overline{R}_2\cup \dots \overline{R}_{N-1}$
\end{enumerate}
\end{defn}

As there are only a finite number of sets in the partition, every set is denoted by a symbol from a finite alphabet and thus the partitioning process for a dynamical system could also be viewed as identification of a many-to-one projection mapping such that the continuous dynamical system is mapped onto a discrete space. More formally, let the continuously varying physical process be modeled as a finite-dimensional dynamical system, $\dot{\mathbf{x}}=f(\mathbf{x}(t))$ where $t\in [0,\infty)$ and $\mathbf{x}\in \real^n$ is the state-vector in the compact phase space $\Omega$ of the system. Then, the partitioning process could be defined as a map $\varphi:\Omega \rightarrow \alphabetSet$   such that $\varphi(\mathbf{x})=\alphabet$ if $\mathbf{x}\in {R}_\alphabet$, where $R_\alphabet \subseteq \Omega$,  $\bigcup\limits_{\alphabet \in \alphabetSet}R_\alphabet=\Omega$.

Let $\alphabetSet$ denote the ordered set of $n$ symbols. The phase space of the symbolic system is the space.
\begin{equation}
X_{[n]}=\alphabetSet^{\mathds Z} =\{a =(a_k)_{k \in \mathds Z}\mid a_k \in \alphabetSet \} 
\end{equation}
of all bi-infinite sequences of elements from a set of $n$ symbols. The shift transformation $\sigma$ is defined by shifting each bi-infinite sequence one step to the left. This is expressed as 
\begin{equation}
(\sigma a)_k=a_{k+1} \nonumber
\end{equation}
The symbolic system $(X_{[n]},\sigma)$ is called the \textit{full n-shift}. Restricting the shift transformation to a full-shift $X_{[n]}$ to a closed shift-invariant subspace $X$, we get a very general kind of dynamical system $(X,\sigma)$ called a \textit{subshift}. Given a symbolic phase space $X$, we call a $k$-tuple an \textit{allowable $k$-block} if it equals $a_{[m,m+k-1]}$ for some sequence $a$. Then, we define a \textit{shift of finite type}, also called the \textit{topological Markov shift}, as the subshift (i.e., shift-invariant) of a full shift restricted to the set $X_G$ of bi-finite paths in a finite directed graph $G$ derived from a full one by possibly removing some edges. The space could also be denoted by $X_{V}$ where $V$ is a matrix and $v_{ij}$ denotes the number of edges going out from node $i$ to node $j$. This gives the resemblance to a Markov chain as normalizing each row of $V$ gives a stochastic matrix which defines a Markov chain over $G$. Before we move to more details of theory of discrete ergodic processes, which provides a statistical view to the discrete stochastic processes, we would like to clarify that a symbolic sequence in a topological Markov shift is bilaterally transitive if every admissible block appears in both directions and infinitely often. In the subsequent section we use the notation $BLT(A)$ to denote the subset of bilaterally transitive points in $A\subseteq M$.\\ 
A (discrete-time, stochastic) process is a sequence $S_1, S_2, \dots, S_n, \dots$ of random variables defined on a probability space $(S,\mathcal{E}, P)$. The process has \textit{alphabet} $\alphabetSet$ if the range of each $S_i$ is contained in $\alphabetSet$. In this paper, we focus on finite-alphabet processes, so, unless otherwise specified, ``process" means a discrete-time finite-alphabet process. Also, unless it is specified otherwise, ``measure" will mean ``probability measure" and ``function" will mean ``measurable function" w.r.t. some appropriate $\sigma$-algebra on a probability space. The cardinality of a finite set $A$ is denoted by $|A|$. The sequence $a_m,a_{m+1},\dots, a_n$, where each $a_i \in \alphabetSet$ is denoted by $a_m^n$; the corresponding set of all such $a_m^n$ is denoted by $\alphabetSet_m^n$ except for $m=1$, when $\alphabetSet^n$ is used. The $k^{th}$ order joint distribution of the process $\{S_k\}$ is the measure $P_k$ on $\alphabetSet^k$ defined by the formula
\begin{equation}
P_k(a_1^k)=\prob(S_1^k=a_1^k), a_1^k\in \alphabetSet^k
\end{equation}

\begin{defn}[\textbf{Stationary Process}] A process is stationary, if the joint distribution do not depend on the choice of the time origin, i.e., 
\begin{equation}
\prob(S_i=a_i, m\leq i\leq n)=\prob(S_{i+1}=a_i, m\leq i \leq n),
\end{equation}
for all $m,\quad n$ and $a_m^n$. 
\end{defn}
These stationary finite-alphabet processes serve as models in many settings of interest like physics, data transmission, statistics, etc.. In this paper, we are mainly interested in statistical and machine learning applications like density estimation, pattern matching, data clustering, and estimation.

The simplest example of a stationary finite-alphabet process is an independent, identically distributed distributed (i.i.d.) process.  A sequence of random variables, $\{S_n\}$ is independent if
\begin{equation}
\prob (S_n=a_n|S_1^{n-1}=a_1^{n-1})=\prob (S_n=a_n)
\end{equation}
holds for all $n \geq 1$ and all $a_1^n \in \alphabetSet^n$. It is identically distributed if $\prob(S_n=a)=\prob(S_1=a)$ holds for all $n \geq 1$ and for all $a \in \alphabetSet$. An independent process is stationary if and only if it is identically distributed. The simplest example of dependent finite-alphabet processes are the Markov processes. 
\begin{defn}[\textbf{Markov Chain}] \label{def:Markovchain} A sequence of random variables, $\{S_n\}$, is a Markov chain if
\begin{equation}
\prob(S_n=a_n|S_1^{n-1}=a_1^{n-1})=\prob(S_n=a_n|S_{n-1}=a_{n-1})
\end{equation}
holds for all $n\geq 2$ and all $a_1^n\in \alphabetSet^n$. A Markov chain is known as homogeneous or to have stationary transitions if $\prob(S_n=b|S_{n-1}=a)$ does not depend on $n$, in which case the $|A| \times |A|$ matrix $M$ defined by 
\begin{equation}
M_{b|a}=\prob(S_n=b|S_{n-1}=a)
\end{equation}
is called the transition matrix for the chain.
\end{defn}

A generalization of the Markov property allows dependence on $k$ steps in the past. A process is called $k-$step  (or $k-$order) Markov if
\begin{equation}
\prob (S_n=a_n|S_1^{n-1}=a_1^{n-1})=\prob (S_n=a_n|S_{n-k}^{n-1}=a_{n-k}^{n-1})
\end{equation}
holds for all $n >k$ and for all $a_1^n$.

The symbolic time-series analysis is initialized by partitioning the phase space of a dynamical system which is a non-linear mapping from a continuous space to a discrete space. In machine learning literature, discretization is generally studied as a feature extraction technique. In dynamical system literature, the partitioning or discretization is characterized by the extent to which a dynamical system can be represented by a symbolic one.

%\begin{defn}(Partition) \label{def:partition} A partition of a non-empty set $X$ is a set of non-empty subsets of $X$ such that every element $x$ in $X$ is in exactly one of these subsets. A family of sets $P$ is said to be a partition of set $X$ iff all of the following conditions hold.
%\begin{itemize}
%\item $\emptyset \not\in P$
%\item $\bigcup\limits_{\mathcal{X}\in P} \mathcal{X}=X$
%\item If $\mathcal{X},\tilde{\mathcal{X}} \in P$ and $\mathcal{X}\neq \tilde{\mathcal{X}}$, then $\mathcal{X}\cap \tilde{\mathcal{X}} = \emptyset$
%\end{itemize}
%\end{defn}
% and the elements of $\alphabetSet$ are mutually disjoint.

The data is partitioned based on the choice of a partitioning technique and then, the dynamics of the discrete process is studied. The dynamics of the symbols sequences are
compressed as a Probabilistic Finite State Automata (PFSA), which is defined
as follows:

\begin{defn}[\textbf{PFSA}]\label{defn:PFSA}
A Probabilistic Finite State Automata (PFSA) is a tuple $\pfsa =\{ \stateSet, \alphabetSet, \trFn, \emMat\}$ where
\begin{itemize}
\item $\stateSet$ is a finite set of states of the automata;
\item $\alphabetSet$ is a finite alphabet set of symbols $\alphabet \in \alphabetSet$;
\item $\trFn: \stateSet \times \alphabetSet \rightarrow \stateSet$ is the state transition function;
%\item $\state_{0}$ is the initial state
\item {$\emMat: \stateSet \times \alphabetSet \rightarrow [0, 1]$ is the
  $\card{\stateSet}\times\card{\alphabetSet}$ emission matrix. The matrix
  $\emMat = [\emProb_{ij}]$ is row stochastic such that $\emProb_{ij}$ is the
  probability of generating symbol $\alphabet_{j}$ from state $\state_{i}$}.
%\item {$\trMat: \stateSet\times\stateSet \rightarrow [0, 1]$ is the
%$\card{\stateSet}\times\card{\stateSet}$ state-transition matrix. The matrix
%$\trMat = [\trProb_{ij}]$ is row stochastic and $\trProb_{ij}$ is the
%probability $\prob(\state_{j}|\state_{i})$ of visiting state $j$ given state
%$i$ in the past.}
\end{itemize}
\end{defn}
\begin{rem} The $\card{\stateSet}\times\card{\stateSet}$ state-transition matrix
  $\trMat: \stateSet\times\stateSet \rightarrow [0, 1]$ can be constructed from
  the state transition function $\delta$ and the emission matrix $\emMat$. The state transition matrix $\trMat$ can be defined as follows.
  \begin{equation}
  \trMat_{jk} \triangleq \sum\limits_{\alphabet \in \alphabetSet:\trFn(\state_j,\alphabet)=\state_k} \emMat(\state_j,\alphabet), \forall \state_j, \state_k\in \stateSet \nonumber
  \end{equation}
  The matrix $\trMat = [\trProb_{ij}]$ is row stochastic and $\trProb_{ij}$ is the
  probability $\prob(\state_{j}|\state_{i})$ of visiting state $j$ given state
  $i$ in the past.
\end{rem}

A PFSA can be seen as a generative machine - probabilistically generating a
symbol string through state transitions.  Ranging from pattern recognition,
machine learning to computational linguistics, PFSA has found many uses. Certain
well known classes of PFSAs are Hidden Markov Models, stochastic regular
grammars and n-grams. A comprehensive review of PFSA can be found
in~\cite{TDC05-1,TDC05-2}. From the perspective of machine learning, the idea is that we infer the regular patterns in the symbol sequence which help explain the temporal dependence in the discrete sequence; the patterns we look for are words of finite length over the alphabet of the discrete process. Once these patterns are discovered, then the discrete data is compressed as a PFSA whose states are the words of finite length; this step leads to some information loss. The parameters of the PFSA are the symbol emission probabilities from each of the finite length words (states of PFSA). Under the assumption of an ergodic Markov chain, we relax requirement for initial states and thus the only \textit{sufficient statistics}~\cite{V04, P13} for the inferred Markov model are the corresponding symbol emission probabilities. For clarification of presentation, we present the following. 
\begin{defn}[\textbf{Sufficient Statistic}]
Let us assume $S_1,S_2,\dots,S_n\sim p(s;\alpha)$. Then any function $\mathcal{T}=\mathcal{T}(S_1,S_2,\dots,S_n)$ is itself a random variable which we will call a \textit{statistic}. Then, $\mathcal{T}$ is sufficient for $\alpha$ if the conditional distribution of $S_1,S_2,\dots,S_n\mid \mathcal{T}$ does not depend on $\alpha$. Thus we have $p(s_1,s_2,\dots,s_n\mid t;\alpha)=p(s_1,s_2,\dots,s_n\mid t).$
\end{defn}
For the case of finite-order, finite-state Markov chains it means that the conditional symbol emission probabilities and the initial state summarizes in entirety the whole of the relevant information supplied by any sample~\cite{L93}. Under stationarity assumptions, the initial state becomes unnecessary and thus, the sufficient statistics is provided by the conditional symbol emission probabilities. The details on estimation of the parameters are provided later in the paper. 

For symbolic analysis of time-series data, a class of PFSAs called the
$\depth$-Markov machine have been proposed~\cite{R04} as a sub-optimal but
computationally efficient approach to encode the dynamics of symbol sequences as
a finite state machine. The main assumption (and reason for sub-optimality) is
that the symbolic process can be approximated as a $D^{th}$ order Markov
process. For most stable and controlled engineering systems that tend to forget
their initial conditions, a finite length memory assumption is reasonable. The
states of this PFSA are words over $\alphabetSet$ of length $\depth$ (or less);
and state transitions are described by a sliding block code of memory $\depth$
and anticipation length of one~\cite{LindMarcus}. The dynamics of this PFSA can
both be described by the $\card{\stateSet}\times\card{\stateSet}$ state
transition matrix $\trMat$ or the $\card{\stateSet}\times 1$ state visit
probability vector $\stProbVec$. Next we present definitions of D-Markov machines which has been recently introduced in literature as finite-memory approximate models for inference and learning.

\begin{defn} (\textbf{$D$-Markov Machine}~\cite{R04, MR14}) \label{def:D-Markov} A $D$-Markov machine is a statistically stationary stochastic process $S= \cdots a_{-1}  a_{0} a_{1} \cdots $ (modeled by a PFSA in which each state is represented by a finite history of $D$ symbols), where the probability of occurrence of a new symbol depends only on the last $D$ symbols, i.e.,
	\begin{equation}\label{eq:D-Markov}
		P[a_n \mid \cdots  a_{n-D} \cdots a_{n-1} ] = P[a_n \mid a_{n-D} \cdots a_{n-1}]
		%P[s_n \mid s_{n-1} \cdots s_{n-D} \cdots] = P[s_n \mid a_{n-1} \cdots a_{n-D}]
	\end{equation}
	where $D$ is called the depth of the Markov machine.

\end{defn} 
A $D-$Markov machine is thus a $D^{th}$-order Markov approximation of the discrete symbolic process.
%\begin{defn} (\textbf{Symbol Block})\label{def:alphabet_symbolBlock} A symbol block, also called a word, is a finite-length string of symbols belonging to the  alphabet $\alphabetSet$, where the length of a word $w \triangleq s_1 s_2 \cdots s_\ell$  with $s_i \in \alphabetSet$ is $|w|=\ell$, and the length of the empty word $\epsilon$ is $|\epsilon|=0$.  
%
%For systems with fading memory it is expected that the predictive influence of a
%symbol progressively diminishes as one goes further away in the future. Formally
%depth is defined as follows:
%\begin{defn}[\textbf{Depth}]\label{def:depth}
%Let $\symbSeq = \symb_{1}\dots\symb_{k}\symb_{k+1}\symb_{k+2}\dots$ be the
%observed symbol sequence  where each $\symb_{j}\in\alphabetSet$ $\forall j \in
%\natno$. Then, the depth of the process generating $\symbSeq$ is defined as the
%length $\depth$ such that:
%\begin{equation}\label{eq:depthTrueDefn}
%\prob(\symb_{k}|\symb_{k - 1},\dots,\symb_{1}) = \prob(\symb_{k}|\symb_{k - 1},\dots,\symb_{k - \depth})
%\end{equation}
%\end{defn}
%The parameters of the PFSA are then extended as follows.
%\begin{itemize}
%\item The set of all words constructed from symbols in $\alphabetSet$, including the empty word $\epsilon$, is denoted as $\alphabetSet^\star$,
%\item The set of all words, whose suffix (respectively, prefix) is the word $w$, is denoted as $\alphabetSet^\star w$ (respectively, $w \alphabetSet^\star$).
%\item The set of all words of (finite) length $\ell$, where $\ell>0$, is denoted as $\alphabetSet^\ell$.
%\end{itemize}

The PFSA model for the time-series data could be inferred as a probabilistic graph whose nodes are words (or symbol blocks) over $\alphabetSet$ of length equal to $\depth$. This probabilistic model induces a Markov model over the states of the PFSA and the parameters of the Markov model can then be estimated from the data using a Maximum-likelihood approach. This completes the process of model inference using STSA. 

Next we define some notions from information and communication theory which will be used in this paper. In particular, we define the notions of Kolmogorov complexity and Minimum Description length and the related ideas.
\begin{defn}(\textbf{Kolmogorov Complexity})\label{def:kolcomplex} The \textit{Kolmogorov Complexity} $K_\mathcal{U}(x)$ of a string $x$ w.r.t. an universal computer $\mathcal{U}$ is defined as 
\begin{equation}
K_\mathcal{U}(x)=\min\limits_{p:\mathcal{U}(p)=x} \ell(p)
\end{equation} 
the minimum length over all programs that print $x$ and halt. Thus, $K_\mathcal{U}(x)$ is the shortest description length of $x$ over all descriptions interpreted by computer $\mathcal{U}$. The term $\mathcal{U}(p)$ denotes the output of the universal computer $\mathcal{U}$, when presented with a program $p$.
\end{defn}
The Minimum Description length (MDL)~\cite{G04} principle is a relatively recent method for inductive inference that provides a generic solution to the model selection problem. MDL is based on the following insight: any regularity in data can be used to compress that data i.e., to describe it using fewer symbols needed to describe the data literally. The more regularities there are, the more the data can be compressed. Thus equating 'learning' with 'finding regularity', we can therefore say that the more we are able to compress the data, the more we have \textit{learned} about the data. 

Next we present Borel-Cantelli lemma, which is a fundamental result in probability theory. This is used in the subsequent sections to establish consistency of the order estimator for Markov chain.
\begin{lem}[\textbf{Borel-Cantelli Lemma}]\label{lemma:BC}
Let $\{B_k\}$ be an arbitrary sequence of events on a probability space $(\Omega,\mathcal{E},P)$ and $p_k=P(B_k)$. Then, the following is true
\begin{equation}
\sum_{k=1}^\infty P(B_k)< \infty \implies P(\lim\sup B_k)=0 \nonumber
\end{equation}
where $\lim\sup$ is defined as follows for any sequence of events $\{E_k\}$: $\lim\sup\limits_{n \rightarrow \infty} \triangleq \bigcap_{n=1}^\infty \bigcup_{k=n}^\infty E_k $.
\end{lem}
\begin{proof}
We have that the following is true.
\begin{equation}
P(\lim\sup A_n)= P(\bigcap_{n=1}^\infty \bigcup_{k=n}^\infty B_k) \leq \lim\limits_{n\rightarrow \infty}\sum_{k>n}p(B_k) \nonumber.
\end{equation}
Since $\sum_{k=1}^{\infty}P(B_k) <\infty$, the tail end of the series must sum to zero. Hence $\lim\limits_{n \rightarrow \infty}\sum_{k>n}P(B_k)=0$. \QEDB
\end{proof}

Next we introduce an information-theoretic distance which is used in the subsequent sections to quantify changes in the Markov models during a physical process.
\begin{defn}[\textbf{Kullback-Leibler Divergence}]\label{def:KLD}
	The Kullback-Leibler (K-L) divergence of a discrete probability distribution $P$ from another distribution $\tilde{P}$ is defined as follows.
	\begin{equation}
		D_{\textrm {KL}}(P\|\tilde{P})=\sum_{x\in X} {p}_X(x)\frac{\log({p}_X(x))}{\log(\tilde{p}_X(x))} \nonumber
	\end{equation}
	It is noted that K-L divergence is not a proper distance as it is not symmetric. However, to treat it as a distance it is generally converted into symmetric divergence as follows, $d(P,\tilde{P})= D_{\textrm {KL}}(P\|\tilde{P})+D_{\textrm {KL}}(\tilde{P}\|P)$. This is defined as the K-L distance between the distributions $P$ and $\tilde{P}$.
\end{defn}

%A reference pattern (e.g. obtained during
%training or the nominal case) can be compared with an observed pattern for
%anomaly detection~\cite{R04}, or these patterns can be used for clustering and
%classification of time-series~\cite{MRJ12}. The alphabet size or the level of
%coarse-graining of the continuous domain is driven by the resolution level
%required to capture the dynamics of the system -- domain knowledge or
%data-driven partitioning techniques~\cite{LHTD2002, GLSLH2012} can be used for
%this purpose. Estimating the depth of historical influences, on the other hand,
%requires estimation of the decay-rate of the memory of a dynamical system.
\section{Problem Formulation}\label{sec:probform}
As explained in section~\ref{sec:background}, there are two critical steps for Markov modeling of data using symbolic dynamics. Once the model structure is fixed, we need to estimate the parameters of the model from the data and thus, the following steps are followed during model inference.
\begin{itemize}
\item Symbolization or Partitioning
\item Order estimation of the discrete stochastic process
\item Parameter estimation of the Markov model inferred with the composite method of discretization and order estimation.
\end{itemize}

The overall idea of modeling the signal as a Markov chain is also represented in Figure~\ref{fig:split_fig} where the various factors to be considered during Markov modeling are emphasized. For example, how the original time-series data is related to the discrete symbol sequence and then, how can we define a metric to relate the representation in the original space and the discrete space. Another important question is to analyze the relation of the original data with the discrete Markov model inferred using the parameters estimated during the modeling process.

\begin{figure}[h] %Figure 03
    \centering
        \includegraphics[width=1.0\columnwidth]{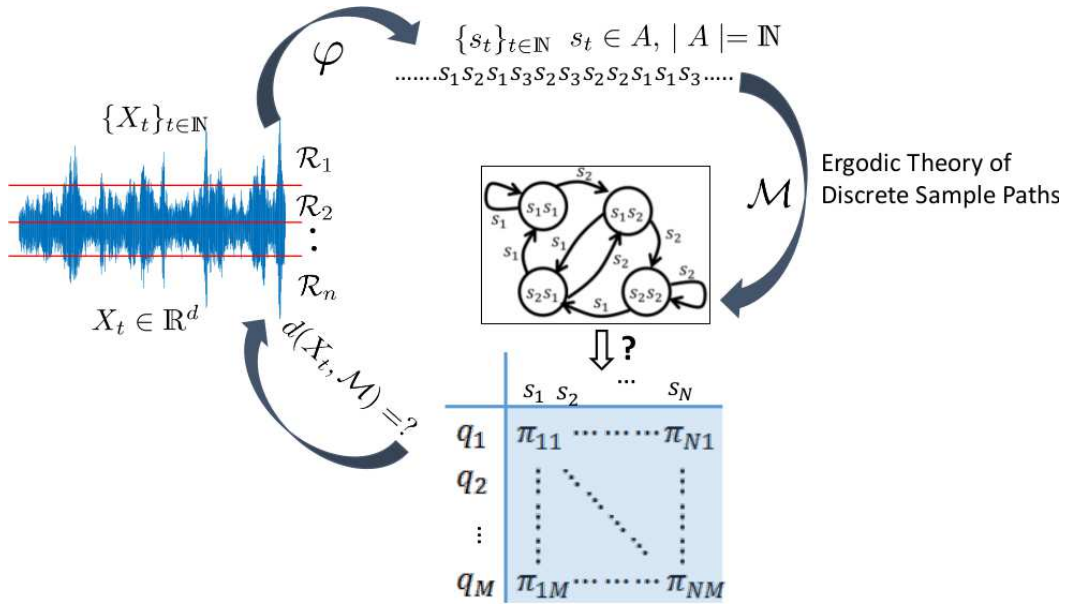}
	\caption{Markov modeling of time series data and the associated decisions of partitioning and parameter estimation%. Final set of states $Q = \{00, 010, 110, 01, 11\}$
	}\label{fig:split_fig}
\end{figure}

In general, for most of the machine learning applications, no description of the underlying system model is available. This in turn makes evaluation of different models for signal representation very difficult as the topological structure of the underlying dynamical system is unknown. In the presence of some end objectives like class separability, anomaly detection, etc. a Bayesian inference rule could be used for selection of the optimal model~\cite{F04}. However, the problem is ill-posed in the absence of a well defined end objective when the aim is to just infer the precise model for signal representation which can be later used for various different operations like clustering or classification. The problem of partitioning the data is particularly ill-posed in the absence of a model, as it is difficult to specify a metric to define a consistent criterion as the process terminates with a Markov model of the data when the true statistical nature of data is unknown. In general, the order of the discrete data depends on the parameters of the partitioning process i.e., the cardinality of the discrete set and the location of the partitions. Thus, although, the precision of the Markov model depends on the partitioning process but in the absence of any model for the data, it is hard to evaluate performance or find a criterion for consistency.

In this paper, we try to present some results available in open literature on the mathematical formalism of partitions and order estimation of discrete stationary Markov processes. In particular, we try to answer the following questions.
\begin{itemize}
\item \textbf{Problem 1.} Given a compact phase space $\Omega \subseteq \real^n$ of a dynamical system $\dot{\mathbf{x}}=f(\mathbf{x})$, under what conditions a topological Markov shift representation for the dynamical system yield a partitioning for the system? Furthermore, how are such partitions characterized?
\item \textbf{Problem 2.} Synthesis of an estimator for estimating the order of a discrete stochastic process that estimates the right order of the Markov process with probability $1$.
\item \textbf{Problem 3.} For practical applications, it is more suitable if the parameters of the final Markov model could be estimated in a recursive manner. How to synthesize the recursive update rule for the Markov model? 
\end{itemize}

%Some other problems that are of interest emphasize on the computational aspects and robustness of the underlying model inference. In particular, the models constructed from a data-driven perspective should be robust to process and measurement noise. In model-based approaches, it is common to assume some parametric uncertainties in the model parameters which can then be used to arrive at a robust decision. Similarly, it is important to associate uncertainties in the model learned from data to account for change in hidden variables for the process. In data-driven systems, it is a common approach to learn the system model from an ensemble of training data; this necessitates a methodological approach to represent the ensemble as models for different instances could be widely different. The problem that we are interested in could then be stated as follows: Let us assume that a random observation $Y \in \real$ under a particular event of interest admits a model $M$ as the nominal model. However, the actual model $\tilde{M}$ of $Y$ under the event is not known exactly and belongs to the neighborhood
%\begin{equation}
%\mathcal{M}=\{\tilde{M}:d(M,\tilde{M})\leq \epsilon\}
%\end{equation}
%where, $d$ is an appropriate metric. Then, the problem of interest is to delineate a framework to represent the set $\mathcal{M}$ during training and how such information could be used to arrive at robust decisions during testing.

In the following sections, we try to answer the above questions and then, demonstrate the proposed methods of data-sets from some engineering systems. We will also present how the two steps of discretization and order estimation can be tied together for inferring a Markov model from the time series where we try to explore the interplay between discretization and order estimation.
\section{Partitioning of Phase Space}\label{sec:partitions}
In this section, we will present some concepts and results on some special kind of well-behaved partitioning. Most the results in this section are based on earlier work presented in~\cite{A98}~\cite{LindMarcus}. Specifically, we will focus on a special class of partitions called \textit{Markov partitions}. Loosely speaking, by using a Markov partition, a dynamical system can be made to resemble a discrete-time Markov process, with the long-term dynamical characteristics of the system represented by a Markov shift. Before providing a formal definition for Markov partitions we need to visit some properties of a partition from dynamical systems literature. For brevity proofs of some simple properties are skipped. Interested readers are referred to~\cite{A98} for more details.  It is noted that in the subsequent material interior of an arbitrary set $\X$ is denoted by $\X^o$.
\begin{defn}\label{def:toprefine}
Given two topological partitions $\mathcal{R}=\{R_0,R_1,\dots,R_{n-1}\}$ and $\mathcal{P}=\{P_0,P_1,\dots,P_{n-1}\}$, we define their topological refinement $\mathcal{R} \lor \mathcal{P} $ as 
\begin{equation}
\mathcal{R}\lor \mathcal{P}= \{R_i\cap P_j: R_i \in \mathcal{R}, P_j \in \mathcal{P} \} \nonumber
\end{equation}
\end{defn}
Based on the definition~\ref{def:toprefine}, the following are true.
\begin{itemize}
\item The common topological refinement of two topological partitions is a topological partition.
\item For a dynamical system $(M,\phi)$ with topological partition $\mathcal{R}$ of $M$, the set defined by $\phi^n \mathcal{R}= \{ \phi^n R_0,\phi^n R_1,\dots, \\ \phi^n R_{n-1} \}$ is also a topological partition. 
\end{itemize}
From the above two results we have that for $m \leq n$, $\bigvee_{m}^n \phi^k \mathcal{R}=\phi^m \mathcal{R} \lor \phi^{m-1}\mathcal{R}\lor \dots \lor \phi^n \mathcal{R}$ is again a topological partition. In the subsequent material, we use the following notations. 
\begin{equation}
\mathcal{R}^n \equiv \bigvee\limits_{k=0}^{n-1} \phi^{-k} \mathcal{R} \nonumber
\end{equation}
The diameter of a partition is defined as $d(\mathcal{R})=\max\limits_{R_i \in \mathcal{R}}d(R_i)$, where $d(R_i)\equiv \sup\limits_{x,y \in R_i}d(x,y)$. Next, we define cylinder set partition for symbol sequences which will be useful to characterize sufficient conditions for existence of Markov partitions for a dynamical system. Consider a topological Markov shift $(X_V,\sigma)$ where the vertex set is labeled by the alphabet $\alphabetSet$. Then, we form the partition $\mathcal{C}=\{C_a : a\in \alphabetSet \}$ of the elementary cylinder sets determined by fixing the $0th-$coordinate i.e., $\mathcal{C}_a= \{x \in X_V: x_0=a \}$. Then, this has the following properties. 
\begin{enumerate}
\item $\bigcap_{m=0}^\infty\bigcap_{-m}^m \phi^{-k}C_{a_k}=\{x\}$.
\item If $x \in C_a \cap \sigma^{-1}C_b$, then $x \in C_a$ and $\sigma x \in C_b$, i.e., $x_0=a,x_1=b$ (or it means an edge from $a$ to $b$ in terms of a graph.) 
\item If $C_a \cap \sigma^{-1} C_b \neq \emptyset$ and $C_b \cap \sigma^{-1} C_c \neq \emptyset$, then $C_a \cap \sigma^{-1} C_b \cap \sigma^{-2} C_c \neq \emptyset$. This can be extended to length $n$ for arbitrary $n$. We shall call such a countable set of conditions for $n=2,3, \dots$ the Markov property: this turns out to be one of the key properties in getting the desired symbolic representation from a partition. 
\end{enumerate}
\begin{defn}[\textbf{Generator}]
We call a topological partition a generator for a dynamical system $(M,\phi)$ if $\lim\limits_{n \rightarrow \infty}d(\bigvee_{-n}^n\phi^k \mathcal{R})=0$.
\end{defn}
For an expansive dynamical system $(M,\phi)$ let $\mathcal{R}$ be a topological partition such that $d(\mathcal{R})<c$ where $c$ is the expansive constant. Then, $\mathcal{R}$ is a generator.\\

For a topological partition $\mathcal{R}$, let us call $X_{\mathcal{R},\phi}$ the symbolic dynamical representation of $(M,\phi)$. Then, for each $x \in X_{\mathcal{R},\phi}$ we define $E_n(x) = \bigcap\limits_{k=-n}^n \phi^{-k}R_{x_k}$. Markov partitions are defined more formally next.
\begin{defn}[\textbf{Markov Partitions}]\label{def:MP}
Let $(M,\phi)$ define be an invertible dynamical system. A topological partition $\mathcal{R}=\{R_1,R_2,\dots,R_{n-1}\}$ of $M$ gives a symbolic representation of $(M,\phi)$ if for every $x\in X_{\mathcal{R},\phi}$ the intersection of $\cap_{n=0}^\infty\bar{E}_n(x)$ contains exactly one point. We call $\mathcal{R}$ to be a Markov partition for $(M,\phi)$ if $\mathcal{R}$ gives a symbolic representation of $(M,\phi)$ and furthermore, $X_{\mathcal{R},\phi}$ is a shift of finite type.
\end{defn}
There are natural extensions of Markov partitions for the cases where the dynamical system is not invertible where we define one-sided symbolic representation of $(M,\phi)$ which is denoted as $X_{\mathcal{R},\phi}^+$. Next we present a theorem which specifies sufficient condition for existence of Markov partitions. In the following we prove that the partition $\mathcal{C}=\{C_i: i=0,1,\dots, n-1$ consisting of elementary cylinder sets $C_i= \{x \in X_V: x_0=i\} $ for a dynamical system $(X_V,\sigma)$ is a topological Markov generator.
%\begin{thm}\label{Markovpartition-factormap}
%Suppose the dynamical system $(M,\phi)$ is expansive and has a Markov generator $\mathcal{R}=\{R_1,R_2,\dots R_{n-1}\}$. Then, the map $\pi$ is an essentially one-to-one finite factor map of the shift of finite type $\Sigma_G$ onto $M$. Furthermore, if $(M,\phi)$ is irreducible, then so is $(\Sigma_G,\psi)$.
%\end{thm}
%\begin{proof}
%We need to show that the following properties holds true for $\pi$ in order to establish that it is a factor map of $\Sigma_G$ onto $M$. 
%
%In order to establish $(iv)$, a bound on the number of pre-images under $\pi-$ we introduce the following concept.
%\begin{defn}
%A map $\pi$ from $\Sigma_G$ to $M$ is called a diamond if there are two sequences $s,t \in \Sigma_G$ for which $\pi(s)=\pi(t)$ and for which there exist indices $k <l<m$ such that $s_k=t_k$, $s_l\neq t_l$ and $s_m=t_m$. 
%\end{defn} \QEDB
%\end{proof}
\begin{thm}~\cite{A98}\label{thm:Markovpartition}
Let $(M,\phi)$ be a dynamical system, $(X_G,\pi)$ an irreducible shift of finite type based on $n$ symbols, and suppose there exists an essentially one-to-one factor map $\pi$ from $X_G$ to $M$. Then the partition $\mathcal{R}=\{R_i=\pi(C_i)^o:i=0,1,\dots,n-1\}$ is a topological Markov partition.
\end{thm}
\begin{proof}
We need to prove the following in order to prove that $\mathcal{R}$ is the Markov partition. 
\begin{enumerate}
\item Elements of $\mathcal{R}$ are disjoint.
\item The closure of elements of $\mathcal{R}$ cover $M$.
\item $R$ is a generator.
\item $\mathcal{R}$ satisfies the Markov property.
\end{enumerate}
In the following we prove the four items listed above in order to prove that $\mathcal{R}$ is a Markov partitions.
\begin{enumerate}
\item Elements of $\mathcal{R}$ are disjoint, i.e., $R_i\cap R_j =\emptyset$, $i \neq j$. \\
The idea of the proof is to use bilaterally transitive points to overcome a difficulty that, in general, maps do not enjoy the property that the image of an intersection is equal to the intersection of images, but one-to-one maps do. Suppose $R_i\cap R_j \neq 0$ for $i \neq j$. Then, by proposition .. $BLT (R_i\cap R_j)\neq \emptyset$. By proposition $4.8$ and $4.9$, $\pi^{-1}$ maps $BLT(R_i)$ and $BLT(R_j)$ homeomorphically onto $BLT(C_i)$ and $BLT(C_j)$ respectively. Therefor $\pi^{-1}$ maps $BLT(R_i\cap R_j)= BLT(R_i)\cap BLT(R_j)$ homeomorphically onto $BLT(C_i\cap C_j)=BLT(C_i)\cap BLT(C_j)$, which implies that $\emptyset \neq BLT(C_i \cap C_j)\subset C_i\cap C_j$, which is a contradiction. \\
\item $M=\cup_{i=0}^{n-1} \bar{R}_i.$\\
$M =\pi(\Sigma_G)$ = $\pi(\cup_{i=0}^{n-1}C_i)$ = $\cup_{i=0}^{n-1}\pi(C_i)$ = $\cup_{i=0}^{n-1}\bar{R}_i$, the last equality follows  from proposition ... \\

To prove the next two items, we need another lemma which is presented next.
\begin{lem}
Under the hypothesis of Theorem~\ref{thm:Markovpartition}, $\pi (\cap_m^n\sigma^{-k}C_{a_k})=\overline{\cap_m^n \phi^{-k}R_{a_k}}$.
\end{lem}
\begin{proof}
Once again we use the bilaterally transitive points to deal with images of intersections. We have the following equalities.
\begin{equation}
\pi(\bigcap_m^n \sigma^{-k}C_{a_k}) =\pi \big( \overline{BLT(\bigcap_m^n \sigma^{-k}C_{a_k})}\big) 
= \overline{\pi (\bigcap_m^n BLT(\sigma^{-k}C_{a_k})\big)} \nonumber
\end{equation}
which by injectivity of $\pi$ and shift-invariance of bilateral transitive points.
\begin{equation}
=\overline{\bigcap_m^n \pi(BLT(\sigma^{-k}C_{a_k}))} = \overline{\bigcap_m^n \pi \sigma^{-k}BLT (C_{a_k})} \nonumber
\end{equation}
which by commutativity of $\pi$ can be reduced to the following.
\begin{equation}
=\overline{\bigcap_m^n \phi^{-k} \pi (BLT(C_{a_k}))}=\overline{\bigcap_m^n\phi^{-k}(BLT(\pi(C_{a_k})))} \nonumber
\end{equation}
This can be further reduced to the following.
\begin{equation}
= \overline{\bigcap_m^n\phi^{-k}(BLT(R_{a_k}))}=\overline{BLT(\bigcap_m^n\phi^{-k}(R_{a_k}))} =\overline{\bigcap_m^n \phi^{-k}(R_{a_k})} \nonumber.
\end{equation} \QEDB
\end{proof}
\item $\mathcal{R}$ is a generator. \\
Because $\mathcal{C}$ is a generator, $d(\cap_{-n}^n\sigma^{-k}C_{a_k})\rightarrow 0$. By Lemma, $\pi(\cap_{-n}^n\sigma^{-k}C_{a_k})=\overline{\cap_{-n}^n\phi^{-k}R_{a_k}} $. So, by continuity of $\pi$, we get.
\begin{equation}
d\big(\overline{\cap_{-n}^n\phi^{-k}R_{a_k}}\big) =d \big(\cap_{-n}^n\phi^{-k}R_{a_k}\big)\rightarrow 0
\end{equation}
This proves that $\mathcal{R}$ is a generator.\\
\item $\mathcal{R}$ satisfies the Markov property. \\
Suppose $R_{a_i}\cap \phi^{-1}R_{a_i+1}\neq \emptyset$, $1\leq k \leq n-1$. By Lemma, we have $\pi[C_{a_i}\cap \sigma^{-1}C_{a_i+1}]= \overline{R_{a_i}\cap \phi^{-1}R_{a_i+1}}\neq \emptyset$, $1\leq k\leq n-1$. Thus $C_{a_i}\cap \sigma^{-1}C_{a_i+1}\neq \emptyset$, $1\leq k \leq n-1$. Since $\mathcal{C}$ satisfies the Markov property, $\cap_{k=1}^n\phi^{-k}C_{a_k}\neq \emptyset$  for all $n>1$. Therefore, $\cap_{k=1}^n\phi^{-k}R_{a_k}\neq \emptyset$ for all $n>1$.
\end{enumerate}
This completes the proof of this theorem. \QEDB
\end{proof}

The above theorem establishes the a sufficient condition for existence of Markov partitions. However, except for some well-behaved dynamical systems, there are no results on existence of Markov partitioning for a dynamical system. Thus, for most of the machine learning applications we are interested in, it would be very difficult to estimate the partitioning which can be proved to induce Markov dynamics. However, given that the dynamics of a discrete stochastic process is Markovian, we need to be able to estimate the order of the Markovian dynamics. This is discussed in the next section.
\section{Order Estimation of Markov Chains}\label{sec:OrderEstimation}
In this section, we will present an order estimation technique for Markov chains which converges almost surely to $k$ given that the underlying stationary and ergodic stochastic process is a $k^{th}$ order Markov chain and infinity otherwise. The work on order estimation presented in this section is based on earlier work presented in~\cite{MW07, MW071, MW08, MW13}. 

Let $\{A\}_{n=0}^\infty$ be a stationary and ergodic time-series taking values from a discrete (finite or infinite) alphabet $\alphabetSet$. In order to estimate the order we need to define some explicit statistics. For $k\geq 0$ let $\mathcal{S}_k$ denote the support of the distribution of $S_{-k}^0$ as
\begin{eqnarray}
\mathcal{S}_k= \{ a_{-k}^0 \in \alphabetSet^{k+1}: p(a_{-k}^0)>0 \} \nonumber
\end{eqnarray}
Next we define
\begin{eqnarray}
\Delta_k=\sup\limits_{1\leq i} \sup\limits_{(b_{-k-i+1}^0,a)\in \mathcal{S}_{k+i}} |p(a|b_{-k+1}^0)-p(a|b_{-k-i+1}^0 )| \nonumber
\end{eqnarray}
We divide the data segment $S_0^n$ into two parts: $S_0^{\lceil \frac{n}{2}\rceil -1}$ and $S_{\lceil\frac{n}{2}\rceil}^n$. Let $\mathcal{S}_{n,k}^{(1)}$ denote the set of strings with length $k+1$ which appear at all in $S_0^{\lceil \frac{n}{2}\rceil -1}$, i.e., 
\begin{equation}
\mathcal{S}_{n,k}^{(1)}=\{a_{-k}^0 \in \alphabetSet^{k+1}: \exists k\leq t\leq \lceil \frac{n}{2}\rceil -1 : S_{t-k}^t=a_{-k}^0\} \nonumber 
\end{equation}
For a fixed $0 <\gamma <1$, let $\mathcal{S}_{n,k}^{(2)}$ denote the set of strings with length $k+1$ which appear more than $n^{1-\gamma}$ times in $S_{\lceil\frac{n}{2}\rceil}^n$. That is,
\begin{equation}
\mathcal{S}_{n,k}^{(2)}=\{a_{-k}^0 \in \alphabetSet^{k+1}: N \{\lceil\frac{n}{2}\rceil+k\leq t\leq \lceil \frac{n}{2}\rceil -1 : S_{t-k}^t=a_{-k}^0\}> n^{1-\gamma}\} \nonumber 
\end{equation}
where $N\{\bullet\}$ denotes the count of $\bullet$. Let 
\begin{equation}
\mathcal{S}_k^n = \mathcal{S}_{n,k}^{(1)} \cap \mathcal{S}_{n,k}^{(2)} \nonumber
\end{equation}

For the sake of notational convenience, let $C(a\mid b_{-k+1}^0:[n_1,n_2])$ denote the empirical conditional probability of $A_1=a$ given $A_{-k+1}^0=b^0_{-k+1}$ from the samples $(A_{n_1},\dots,A_{n_2})$, that is,
\begin{equation}
C(a\mid b_{-k+1}^0:[n_1,n_2])=\frac{N\{n_1+k\leq k\leq n_2:A_{t-k}^t=(b_{-k+1}^0,a)\}}{\{n_1+k-1\leq t\leq n_2-1:A_{t-k+1}^t=b_{-k+1}^0\}}
\end{equation}
where $0/0$ is defined as $0$.

We define the empirical version of $\Delta_k$ as follows:
\begin{equation}
\hat{\Delta}_k^n=\max\limits_{1\leq i\leq n}\max\limits_{(b_{-k-i+1}^0,a)\in \mathcal{S}_{k+i}^n}\mid C(a\mid b_{-k+1}^0:[\lceil\frac{n}{2}\rceil,n])-C(a\mid b_{-k-i+1}^0:[\lceil\frac{n}{2}\rceil,n])\mid
\end{equation}
Then we observe that by ergodicity, for any fixed $k$, the following is true.
\begin{equation}\label{eqn:orderapprox}
\lim\inf\limits_{n\rightarrow\infty}\hat{\Delta}_k^n\geq \Delta_k
\end{equation}

\begin{defn}[\textbf{Markov Chain Order Estimator}]\label{def:MCorderestimate}
We define an estimate $D_n$ for the order of samples from $S_0^n$ as follows. Let the $0<\beta<\frac{1-\gamma}{2}$ be arbitrary. Set $D_0=0$, and for $n\geq 1$ let $D_n$ be the smallest $0\leq k_n<n$ such that $\hat{\Delta}^n_{k_n}\leq n^{-\beta}$.
\end{defn}
The consistency of the Markov chain estimator defined in Definition~\ref{def:MCorderestimate} is proved in a theorem which is presented next.
\begin{thm}~\cite{MW07}
If the stationary and ergodic discrete time series $\{S_n\}$ taking values from a discrete alphabet happens to be Markov with any finite order then $D_n$ equals the order eventually almost surely, and if it is not Markov with any finite order then, 
$D_n \rightarrow \infty$ almost surely.
\end{thm}

\begin{proof}
If the process is Markov, it is immediate that for all $k$ greater than or equal to the order $\Delta_k=0$. For $k$ less than the order $\Delta_k>0$. If the process is not a Markov chain with any finite order then $\Delta_k>0$ for all $k$. This by~\eqref{eqn:orderapprox} if the process is not Markov then $D_n\rightarrow \infty$ and if it is Markov then $D_n$ is greater or equal to the order eventually almost surely. We have to show that $D_n$ is less or equal the order eventually almost surely provided that the process is a Markov chain.

Assume that the process is a Markov chain with order $k$. Let $n\geq k$. We will estimate the probability of the undesirable event as follows.
\begin{align}\label{eqn:eq-1}
&\prob(\hat{\Delta}_k^n>n^{-\beta}|S_0^{\lceil\frac{n}{2}\rceil})\leq& \nonumber \\ 
&\sum_{i=1}^n\prob(\max\limits_{(b_{-k-i+1,a}\in \mathcal{S}_{k+i}^n)}\mid C(a\mid b_{-k+1}^0:[\lceil\frac{n}{2}\rceil,n])-C(a\mid b_{-k-i+1}^0:[\lceil\frac{n}{2},n])\mid >n^{-\beta}\mid S_0^{\lceil\frac{n}{2}\rceil})&
\end{align}
We can estimate each probability in the sum as the sum of the following two terms:
\begin{align}
&\sum_{i=1}^n\prob(\max\limits_{(b_{-k-i+1,a}\in \mathcal{S}_{k+i}^n)}\mid C(a\mid b_{-k+1}^0:[\lceil\frac{n}{2}\rceil,n])-C(a\mid b_{-k-i+1}^0:[\lceil\frac{n}{2}\rceil,n])\mid >n^{-\beta}\mid S_0^{\lceil\frac{n}{2}\rceil})  & \nonumber\\ 
&\leq \prob(\max\limits_{(b_{-k-i+1}^0,a)\in \mathcal{S}_{k+i}^n}\mid C(a\mid b_{-k+1}^0:[\lceil \frac{n}{2}\rceil,n])-p(a\mid b_{-k+1}^0))\mid >0.5n^{-\beta}\mid S_0^{\lceil\frac{n}{2}\rceil})& \nonumber\\ 
&+\prob(\max\limits_{(b_{-k-i+1}^0,a)\in \mathcal{S}_{k+i}^n} \mid p(a\mid b_{-k+1}^0)-C(a\mid b_{-k-i+1}^0:[\lceil\frac{n}{2}\rceil,n])\mid > 0.5n^{-\beta}\mid S_0^{\lceil\frac{n}{2}\rceil}) & \nonumber
\end{align}
We overestimate these probabilities. For any $m\geq 0$ and $a_{-m}^0$ define $\sigma_i^m(a_{-m}^0)$ as the time of the $i-$th occurrence of the string $a_{-m}^0$ in the data segment $S_{\lceil\frac{n}{2}\rceil}^n$, that is, let $\sigma_0^m(a_{-m}^0)=\lceil\frac{n}{2}\rceil+m-1$ and $i\geq 1$ define,
\begin{equation}
\sigma_i^m(a_{-m}^0)=\min \{t >\sigma_{i-1}^m(a_{-m}^0):S_{t-m}^t:a_{-m}^0\} \nonumber
\end{equation}
Then,
\begin{align}
& \prob(\max\limits_{(b_{-k-i+1}^0,a)\in \mathcal{S}_{k+i}^n}\mid C(a\mid b_{-k+1}^0: [\lceil\frac{n}{2}\rceil,n])-C(a \mid b_{-k-i+1}^0:[\lceil\frac{n}{2},n])\mid >n^{-\beta}\mid S_0^{\lceil\frac{n}{2}\rceil})& \nonumber\\
&\leq \prob(\max\limits_{(b_{-k+1}^0,a)\in \mathcal{S}_{n,k}^{(1)}}\sup\limits_{j>n^{1-\gamma}}\mid \frac{1}{j}\sum_{r=1}^j1_{\{S_{\sigma_r^{k-1}(b_{-k+1}^0)}=a\}}-p(a\mid b_{-k+1}^0)\mid >0.5n^{-\beta}\mid S_0^{\lceil\frac{n}{2}\rceil}) &\nonumber\\
&+\prob(\max\limits_{(b_{-k-i+1}^0,a)\in \mathcal{S}_{n,k+i}^{(1)}}\sup\limits_{j>n^{1-\gamma}}\mid \frac{1}{j}\sum_{r=1}^j1_{\{S_{\sigma_r^{k+i-1}(b_{-k-i+1}^0)}=a\}}-p(a\mid b_{-k+1}^0)\mid >0.5n^{-\beta}\mid S_0^{\lceil\frac{n}{2}\rceil})& \nonumber
\end{align}
Since both $S_{n,k}^1$ and $S_{n,k+i}^1$ depend solely on $S_0^{\lceil\frac{n}{2}\rceil}$, we get
\begin{align}
&\prob\big(\max\limits_{(b_{-k-i+1}^0,a)\in S_{k+i}^n} \mid C(a\mid b_{-k+1}^0:[\lceil\frac{n}{2}\rceil,n])-C(a\mid b_{-k-i+1}^0:[\lceil\frac{n}{2}\rceil,n])\mid >n^{-\beta}\mid S_0^{\lceil\frac{n}{2}\rceil}\big)&\nonumber \\
&\leq \sum_{(b_{-k+1}^0,a)\in S_{n,k}^{(1)}}\sum_{j=\lceil n^{1-\gamma}\rceil}^{\infty} \prob(\frac{1}{j}\sum_{r=1}^j 1_{\{S_{\sigma_r^{k-1}(b_{-k+1}^0)=a}\}}-p(a\mid b_{-k+1}^0)\mid>0.5n^{-\beta}\mid S_0^{\lceil\frac{n}{2}\rceil})&\nonumber \\
&+\sum_{(b_{-k-i+1},a)\in \mathcal{S}_{n,k+i}^{(1)}}\sum_{j=\lceil n^{1-\gamma}\rceil}^\infty \prob(\mid \frac{1}{j}\sum_{r=1}^j1_{\{S_{\sigma_{r}^{k+i-1}}(b_{-k-i+1}^0)=a\}}-p(a\mid b_{-k+1}^0)\mid >0.5n^{-\beta}\mid S_0^{\lceil\frac{n}{2}\rceil})&\nonumber
\end{align}
Each of these terms represents the deviation of an empirical count from its mean. The variables under consideration are independent since whenever the block $b_{-k+1}^0$ occurs the next term is chosen using the same distribution $p(a\mid b_{-k+1}^0)$. Thus by Hoeffding's inequality for sums of bounded independent random variables and since the cardinality of both $\mathcal{S}_{n,k}^{(1)}$ and $S_{n,k+i}^{(1)}$ is not greater than $(n+2)/2$, we have the following.
\begin{align}
&\prob\big( \max\limits_{(b_{-k-i+1}^0,a)\mathcal{S}_{k+i}^n}\mid C(a\mid b_{-k+1}^0:[\lceil \frac{n}{2}\rceil,n])-C(a\mid b_{-k-i+1}^0:[\lceil\frac{n}{2}\rceil,n])\mid > n^{-\beta}\mid S_0^{\lceil\frac{n}{2}\rceil}\big)& \nonumber \\
&\leq 2\frac{n+2}{2}\sum\limits_{j=\lceil n^{1-\gamma}\rceil}^\infty 2e^{-n-2\beta j}& \nonumber
\end{align}
Thus we obtain the following.
\begin{equation}
\prob (\hat{\Delta}_k^n>n^{-\beta}\mid S_0^{\lceil\frac{n}{2}\rceil})\leq n(n+2)4e^{-2n-2\beta +1-\gamma} \nonumber
\end{equation}
Integrating both sides we get
\begin{equation}\label{eqn:finalstep}
\prob (\hat{\Delta}_k^n>n^{-\beta}\leq n(n+2)4e^{-2n-2\beta +1-\gamma}
\end{equation}
The right hand side of equation~\eqref{eqn:finalstep} is summable provided $2\beta+\gamma <1$ and the Borel-Cantelli Lemma (see Lemma~\ref{lemma:BC}) yields that $\prob(\hat{\Delta}\leq n^{-\beta})=1$. Thus $D_n\leq k$ eventually almost surely provided the process is Markov with order $k$. The proof is now complete. \QEDB
\end{proof}

The above theorem establishes a consistent estimator for the order of an unknown discrete Markov process. Once the order of the discrete Markov process is estimated, we estimate the sufficient statistics for the Markov chain from the data. This is described in detail in next section.

\section{Parameter Estimation for Markov Chains}\label{sec:Parameters}
In this section, we present some results on estimation of parameters of the Markov chain once the structure of the chain is inferred after discretization and order estimation. Given a finite-length symbol string over a (finite) alphabet $\alphabetSet$, there exists several PFSA construction algorithms to discover the underlying irreducible PFSA model, such as the D-Markov machines~\cite{R04}\cite{RR06}. Once the order of the discrete data is estimated, the states of a D-Markov machines are memory words of length equal to the order over $\alphabetSet$. The sufficient statistics for the D-Markov machine are the symbol emission probabilities conditioned on the memory words. These statistics could be estimated using a Bayesian approach. In general, the symbol emission probabilities conditioned on the individual memory words can be modeled as a multinomial random variable which can be modeled by a Dirichlet distribution. However, we skip those details here and present a very simple estimation process where every random variable has a uniform prior. 

Let $N(\symb_j|q_k)$ denote the number of times that a symbol $\symb$ is generated from the state $q_k$ as the symbol string evolves. The maximum a posteriori probability (MAP) estimate of the morph probability (see Definition~\ref{defn:PFSA}) for the PFSA $\pfsa$ is estimated by frequency counting as

\begin{equation}\label{eq:MorphProbability_MAP}
\hat{\pi}(\symb_j|q_k) \triangleq \frac{C(\symb_j|q_k)}{\sum_\ell C(\symb_\ell|q_k)}= \frac{1+ N(\symb_j|q_k)}{|\alphabetSet| + \sum_\ell N(\symb_\ell|q_k)}
\end{equation}

The rationale for initializing each element of the count matrix $C$ to $1$ is that if no event is generated at a state $q \in Q$, then there should be no preference to any particular symbol and it is logical to have $\hat{\pi}_{MAP}(q,\symb) = \frac{1}{|\alphabetSet|} \  \forall \symb \in \alphabetSet$ (i.e., the uniform distribution of event generation at the state $q$). The above procedure guarantees that the PFSA, constructed from a (finite-length) symbol string, must have an (elementwise) strictly positive morph map $\pi$. % and  that the state transition map $\delta$ in Definition~\ref{def:FSA} is a total function.

Having computed the probabilities $\hat{\pi}\left(\symb_j | q_k \right)$ for all $\symb_j \in \alphabetSet$ and $q_k \in Q$, the estimated emission probability matrix of the PFSA is obtained as
\begin{equation} \label{eq:PMatrix}
	\widehat{\mathbf{\Pi}} \triangleq \left[\begin{array}{lll}\hat{\pi}\left(\symb_1 \mid q_1 \right) & \dots & \hat{\pi}\left(\symb_{\vert \alphabetSet \vert} \mid q_1 \right) \\ \vdots & \ddots & \vdots \\ \hat{\pi}\left(\symb_1 \mid q_{\vert Q \vert} \right) & \cdots & \hat{\pi}\left(\symb_{\vert \alphabetSet \vert} \mid q_{\vert Q \vert} \right) \end{array}\right].
\end{equation}

The stochastic matrix $\widehat{\mathbf{\Pi}}$, estimated from a symbol string, is then treated as a feature of the times series data that represents the behavior of the dynamical system. These features can then be used for different learning applications like pattern matching, classification, and clustering. The estimated stochastic matrices $\widehat{\mathbf{\Pi}}$ are converted into row vectors for feature divergence measurement by sequentially concatenating the $|Q|$ rows (i.e., the $|Q|\times |\Sigma|$ matrix $\widehat{\mathbf{\Pi}}$ is vectorized as a $(1\times |Q||\alphabetSet|)$ vector). The convergence of the parameters of the Markov model could be established by Glivenko-Cantelli theorem~\cite{V98, V13} which guarantees the convergence of empirical distribution of random variables under independent observations. 

\begin{rem}
As compared to the Hidden Markov model-based approaches, this approach presents a much simpler modeling technique where the structure of the model is inferred based on the observations. As the structure of the Markov model is also fixed by the observations, this makes inference of the parameters of the Markov chain much more easier than the standard Hidden Markov model approaches like the Viterbi algorithm~\cite{B06} where Dynamic programming is used to estimate the parameters that maximize the likelihood of the observations.
\end{rem}

The sufficient statistics for the inferred Markov process serve as a compact stochastic representation of the signal and can be used for different machine learning applications like pattern matching, classification, clustering, anomaly detection. As explained earlier, the symbol emission probabilities conditioned on the memory words for the discrete process could be modeled as random variables. A metric to measure the information gain over the marginal symbol emission probabilities could be defined as follows:
\begin{equation}\label{eqn:discrepancy}
\mathfrak{d}_\theta = \sum_{q\in Q} \prob(q) D_{KL}(P(\mathcal{A}\mid q)\| \tilde{P}({\mathcal{A}}))
\end{equation}
where, $\theta$ represents the parameters of the Markov model inference, i.e., the partitioning map and the estimated order of the memory words. The term in equation~\eqref{eqn:discrepancy} measures the measures the discrepancy in the statistics of the symbol emission probabilities when they are conditioned on the memory words as compared to the unconditional statistics for the symbol emission probabilities. 

\section{Statistical Learning Applications}
In this section, we present a case study using the present framework for statistical learning for applications like anomaly detection, classification and prognostics in a complex engineering system. While there are several examples presented in open literature~\cite{jha2018symbolic, li2016information, sarkar2015dynamic, seto2016data, virani2019sequential, jha2013classification, virani2013dynamic, li2017information, li2015feature}, we will only show one example here for clarity. We present an example of complex thermoacoustic instability in gas turbine engines during combustion. Combustion instability is a highly nonlinear and complex phenomena which results in severe structural degradation in jet turbine engines. Some good surveys on the current understanding of the mechanisms for the combustion instability phenomena could be found in~\cite{OAL15, SSDC03, CDSBM14, HY09, MBDSC12}. Active combustion instability control (ACIC) with fuel modulation has proven to be an effective approach for reducing pressure oscillations in combustors~\cite{BMJK06, BMH07}. Based on the work available in literature, one can conclude that the performance of ACIC is primarily limited by the large delay in the feedback loop and the limited actuator bandwidth ~\cite{BMJK06, BMH07}. From the perspective of active control of the unstable phenomena, it is necessary to accurately detect and, desirably, predict the states of the combustion process. In this paper, we present Markov modeling of pressure time-series during combustion and present results on changes in the underlying model structure and complexity as the process undergoes some changes in its physical properties. The goal is to be able to design a statistical model for the instability phenomenon during combustion which could be used to design a statistical filter to accurately predict with high confidence the system states. This can potentially alleviate the problems with delay in the ACIC feedback loop and thus possibly improve the performance. We first present some experimental details of the set-up that was used to collect the experimental data and then, show the Markov modeling of the pressure time-series data which is able to capture the changes in system behavior.

\subsection{Experimental Details}\label{sec:experiment}

In this section we present the experimental details for collecting data to analyze the complex non-linear phenomena that occurs during the instability phenomena, in a laboratory-scale combustor. A swirl-stabilized, lean-premixed, laboratory-scale combustor was used to perform the experimental study. Figure~\ref{fig:testrig} shows a schematic drawing of the variable-length combustor. The combustor consists of an inlet section, an injector, a combustion chamber, and an exhaust section. The combustor chamber consists of an optically-accessible quartz section followed by a variable length steel section.

\begin{figure} %Fig02
	\centering \vspace{-6pt}
	\includegraphics[width=1.0\textwidth]{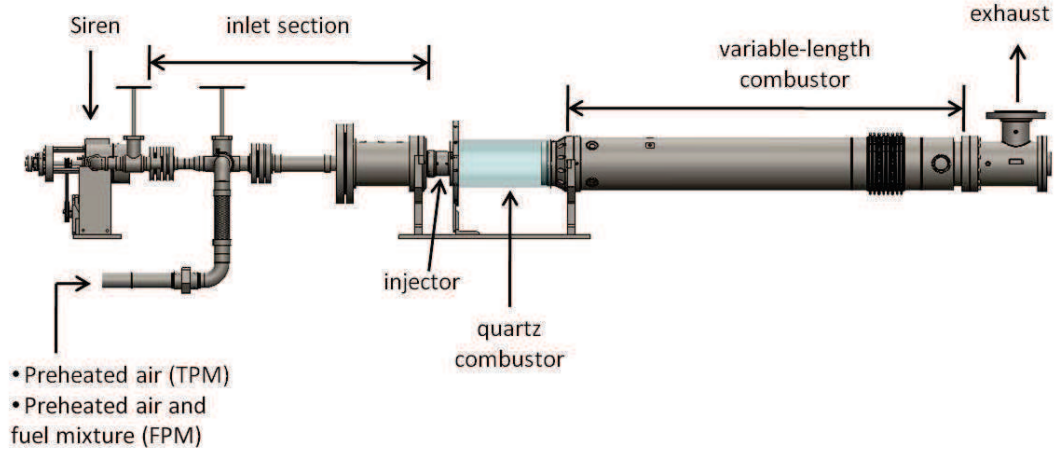}
	\caption{Schematic drawing of test facility}
	\label{fig:testrig}
\end{figure}

\begin{table}[!hbp]
	\centering
	\caption{Operating conditions}
	\begin{tabular}{c|c}
		\hline
		\textbf{Parameters} & \textbf{Value} \\
		\hline
		Equivalence Ratio & 0.525, 0.55, 0.60, 0.65\\
		\hline
		Inlet Velocity & 25-50 m/s in 5 m/s increments \\
		\hline
		Combustor Length & 25-59 inch in 1 inch increments\\
		\hline
	\end{tabular}
	\label{tab:par}
\end{table}

High pressure air is delivered to the experiment from a compressor system after passing through filters to remove any liquid or particles that might be present. The air supply pressure is set to 180 psig using a dome pressure regulator. The air is pre-heated to a maximum temperature of $250\,^{\circ}{\rm C}$ by an 88kW electric heater. The fuel for this study is natural gas (approximately $95$\% methane). It is supplied to the system at a pressure of 200 psig. The flow rates of the air and natural gas are measured by thermal mass flow meters. The desired equivalence ratio and mean inlet velocity is set by adjusting these flow rates with needle valves. For fully pre-mixed experiments (FPM), the fuel is injected far upstream of a choke plate to prevent equivalence ratio fluctuations. For technically pre-mixed experiments (TPM), fuel is injected in the injector section near the swirler. It mixes with air over a short distance between the swirler and the injector exit. Tests were conducted at a nominal combustor pressure of 1 atm over a range of operating conditions, as listed in Table~\ref{tab:par}.
%Fuel injection locations are shown schematically in Figure~\ref{fig:fuelinj}.
%\begin{figure*}[thb] %Fig01
%\centering
%\includegraphics[width=0.5\textwidth]{Figures/fuelinj.eps}
%\caption{Fuel injection locations for fully premixed and technically premixed experiments}
%\label{fig:fuelinj}
%\end{figure*}

 In each test, the combustion chamber dynamic pressure and the global OH and CH chemiluminescence intensity were measured to study the mechanisms of combustion instability. The measurements were made simultaneously at a sampling rate of 8192 Hz~(per channel), and data were collected for 8 seconds, for a total of 65536 measurements~(per channel).

\subsection{Markov Modeling and Results}\label{sec:results}
In this section, we present details of the analyses done using the pressure time-series data to infer the underlying Markov model. Time-series data is first normalized by subtracting the mean and dividing by the
standard deviation of its elements; this step corresponds to bias removal and
variance normalization. Data from engineering systems is typically oversampled
to ensure that the underlying dynamics can be captured. Due to coarse-graining from the symbolization process, an over-sampled
time-series may mask the true nature of the system dynamics in the symbolic
domain (e.g., occurrence of self loops and irrelevant spurious transitions in
the Markov chain). Time-series is first down-sampled to find the next crucial
observation. The first minimum of auto-correlation function generated from the
observed time-series is obtained to find the uncorrelated samples in time. The
data sets are then down-sampled by this lag. To avoid discarding significant
amount of data due to downsampling, down-sampled data using different initial
conditions is concatenated. Further details of this preprocessing can be found
in~\cite{Srivastav2014}.

The continuous time-series data set is then partitioned using maximum entropy
partitioning (MEP), where the information rich regions of the data
set are partitioned finer and those with sparse information are partitioned
coarser. In essence, each cell in the partitioned data set contains
(approximately) equal number of data points under MEP. A ternary alphabet with
$\alphabetSet=\{0,1,2\}$ has been used to symbolize the continuous combustion
instability data. As discussed in section~\ref{sec:experiment}, we analyze data sets from different phases, as the process goes from stable through the transient to the unstable region (the ground truth is decided using the RMS-values of pressure).

In figure~\ref{fig:datadistribution}, we show the observed changes in the behavior of the data as the combustion operating condition changes from stable to unstable. A change in the empirical distribution of data from unimodal to bi-modal is observed as the system moves from stable to unstable. We selected $150$ samples of pressure data from the stable and unstable phases each to analyze and compare. The temporal memory of the individual series is estimated by the spectral decomposition method presented earlier in~\cite{Srivastav2014, JSMR15}. First, we compare the expected size of temporal memory during the two stages of operation. There are changes in the Eigen value decomposition rate for the 1-step stochastic matrix calculated from the data during the stable and unstable behavior, irrespective of the combustor length and inlet velocity. During stable conditions, the Eigen values very quickly go to zero as compared to the unstable operating condition (see Figure~\ref{fig:spectralprop}). This suggests that the size of temporal memory of the discretized data increases as we move to the unstable operating condition. This indicates that under the stable operating condition, the discretized data behaves as symbolic noise as the predictive power of Markov models remain unaffected even if we increase the order of the Markov model. On the other hand, the predictive power of the Markov models can be increased by increasing the order of the Markov model during unstable operating condition, indicating more deterministic behavior. An $\epsilon=0.05$ is chosen to estimate the depth of the Markov models for both the stable and unstable phases. Correspondingly, the depth was calculated as $2$ and $3$ for the stable and unstable conditions (see Figure~\ref{fig:databehavior}). The corresponding $D(\epsilon)$ is used to construct the Markov models next. First a PFSA whose states are words over $\alphabetSet$ of length $\depth(\epsilon)$ is created and the corresponding maximum-likely parameters ($\emMat$ and $\trMat$) are estimated.

\begin{figure*}
	\centering
	\subfloat[Probability density function for the pressure time series data]{\includegraphics[width=0.40\textwidth]{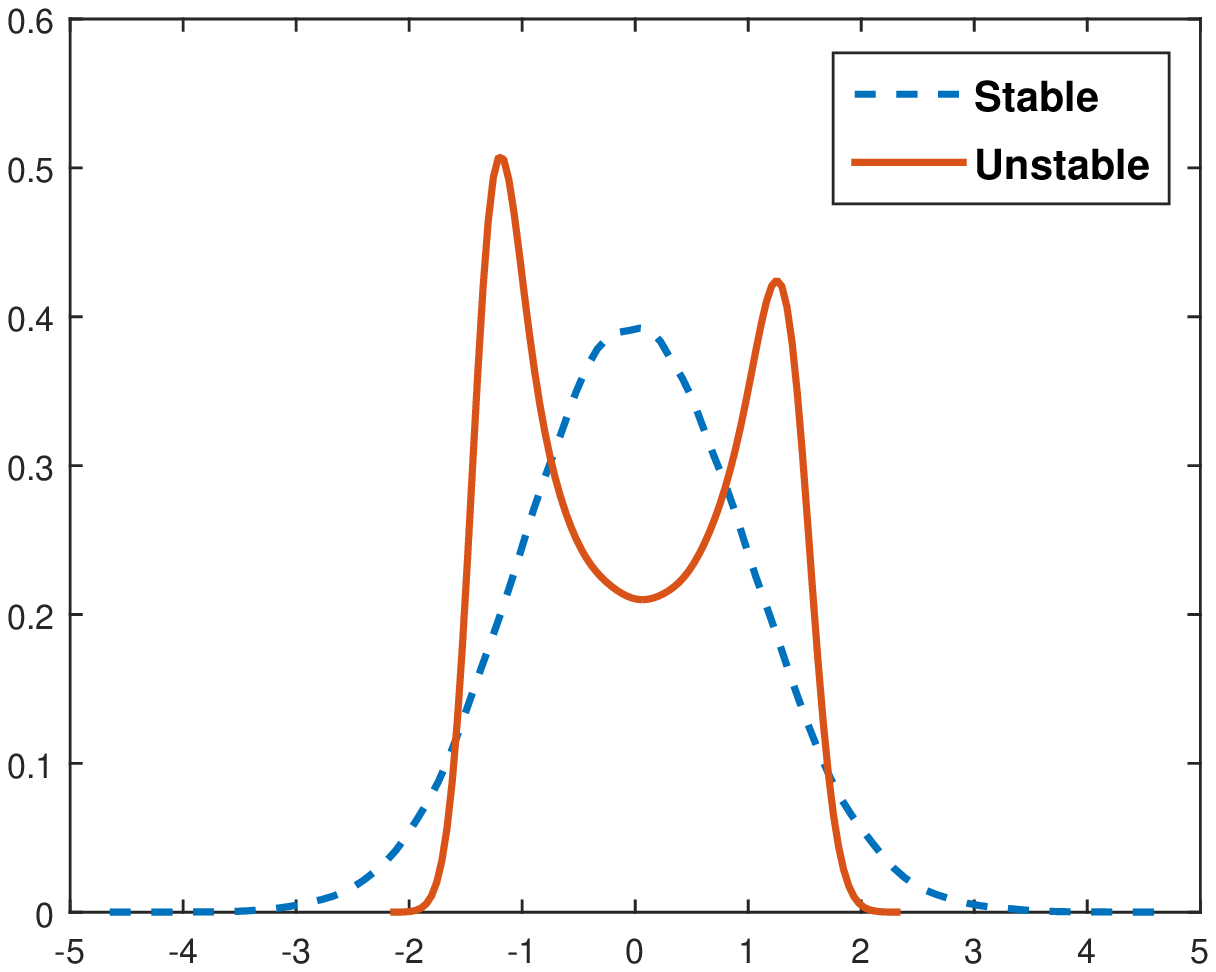}\label{fig:datadistribution}}\quad
	\subfloat[Spectral decomposition of the stochastic matrix for 1-step Markov model]{\includegraphics[width=0.40\textwidth]{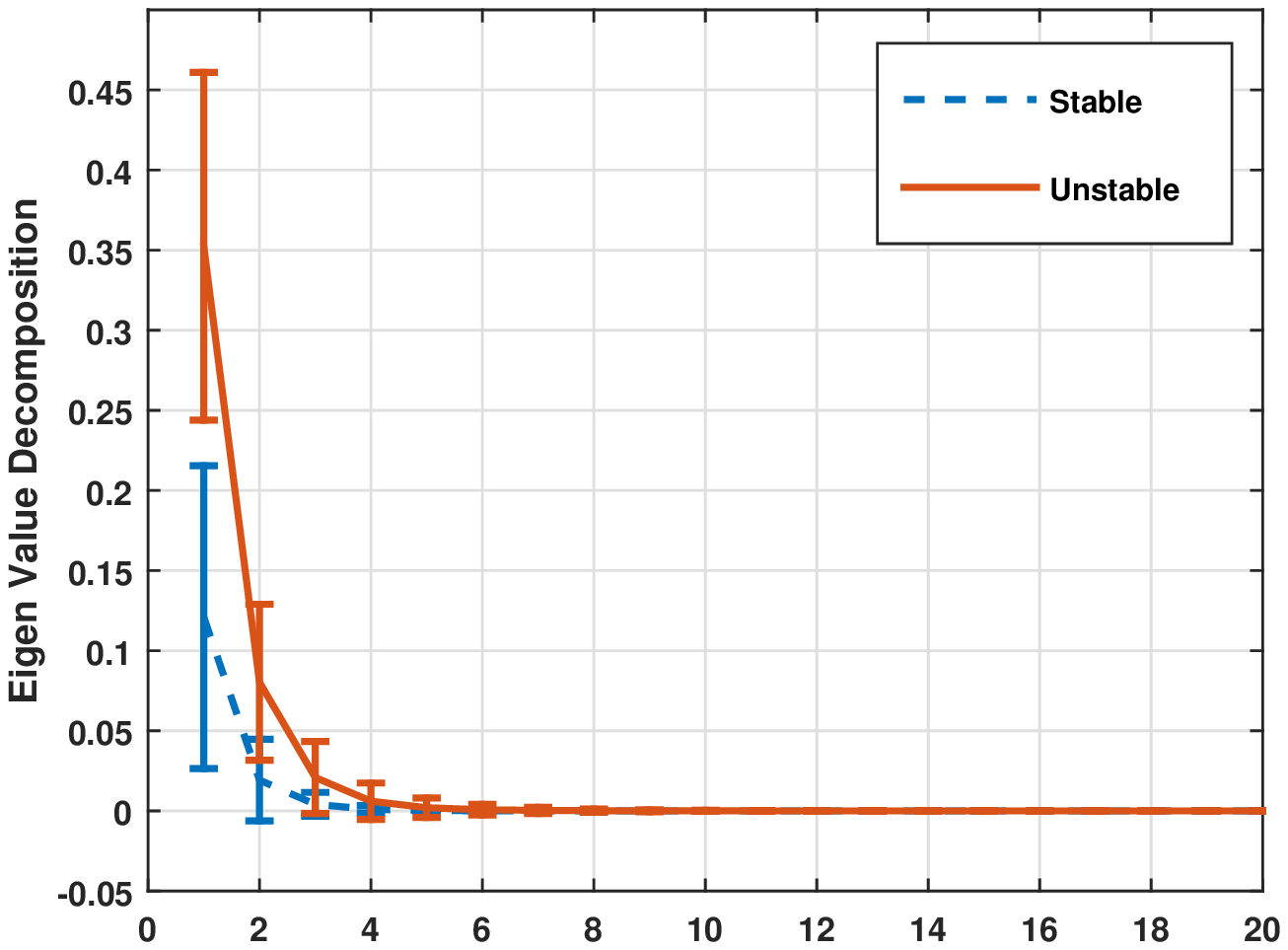}\label{fig:spectralprop}}\\
	\caption{The first plate in the above figure shows the change in the empirical density calculated for the pressure time-series data as the process deviates from the stable operating condition to unstable operating condition. The second plate shows the spectral decomposition of the 1-step stochastic matrix for the data under stable and unstable operating conditions.}
	\label{fig:databehavior}\vspace{-2pt}
\end{figure*}

\begin{figure*}[thb] %Fig02
	\centering \vspace{-6pt}
	\includegraphics[width=0.7\textwidth]{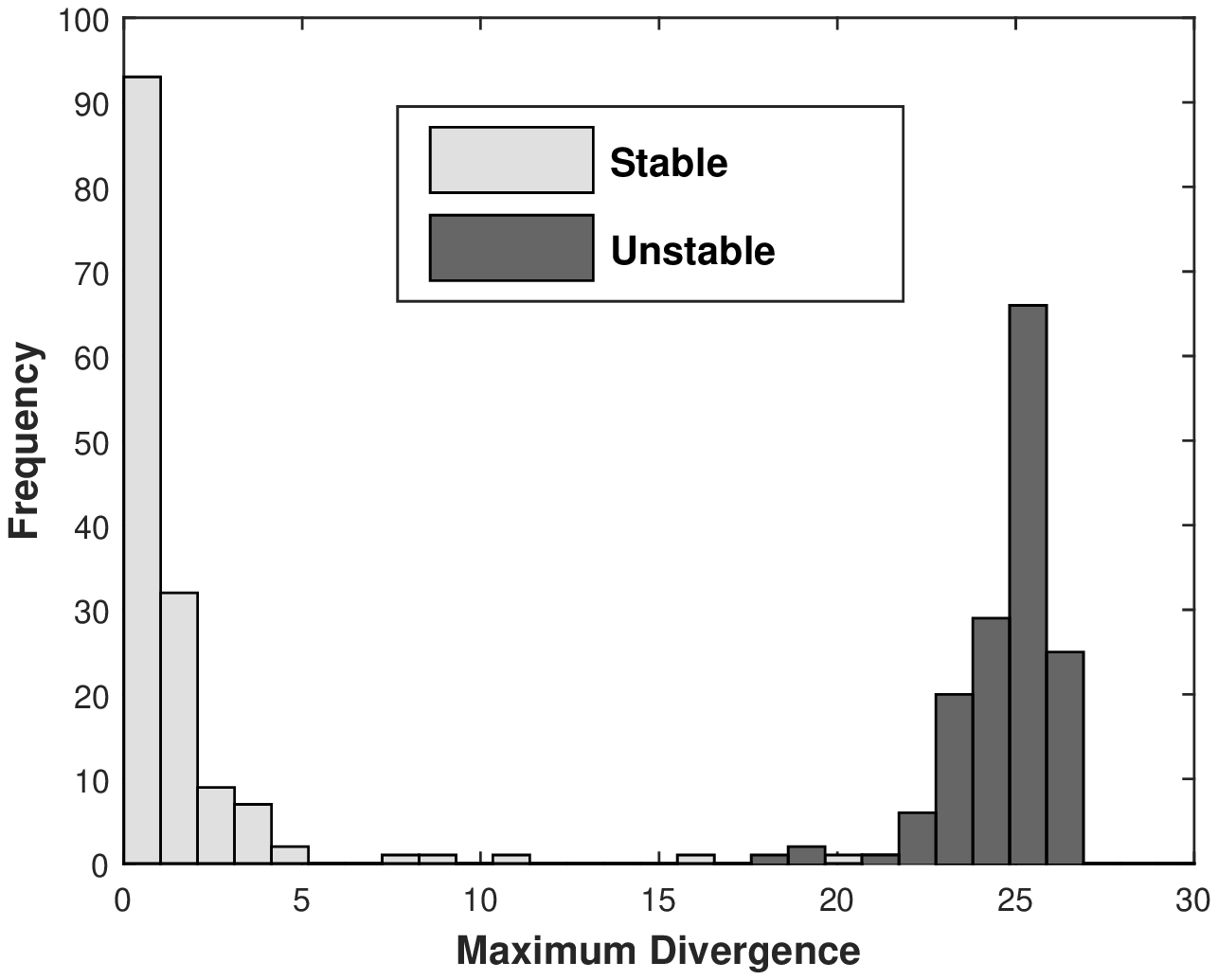}
	\caption{Class-separability of the stable and unstable classes independent of any other related variables.}
	\label{fig:divergence}
\end{figure*}
Once the model structure is inferred and the model parameters are estimated from the data, we estimate certain model-specific metrics to see changes in the inferred models. In particular, we are interested in inferring the changes in the model complexity as the process moves from stable to unstable. The size of temporal memory can also be treated as a metric for complexity of the underlying Markov model. However, we define another metric based on the KL-distance between the states of the Markov model. In particular, we estimate the following. 
\begin{equation}\label{eqn:maxdivergence}
d_{\mathcal{M}}=\max\limits_{q_i,q_j\in\stateSet}d(q_i,q_j) 
\end{equation}
where $d(q_i,q_j)= D_{\textrm{KL}}(\prob(\alphabetSet|\state_i)\|\prob(\alphabetSet|\state_j))+D_{\textrm{KL}}(\prob(\alphabetSet|\state_j)\|\prob(\alphabetSet|\state_i))$, i.e., the symmetric KL distance between states $q_i$ and $q_j$ based on the conditional symbol emission probabilities.  In equation~\eqref{eqn:maxdivergence}, we measure the maximum divergence in the set of states; however, another possible metric could be an expectation over the set of states. The behavior of the metric described in equation~\eqref{eqn:maxdivergence} is shown in figure~\ref{fig:divergence}. It is clear that the proposed metric is able to achieve a clear separation of the two classes of interest (see Figure~\ref{fig:divergence}).

Next we show the gradual change in the model characteristics and compare it with the behavior of the RMS of pressure during this process. It is noted at this point that all the analysis is done in an unsupervised fashion, so this falls under the broad category of unsupervised learning or anomaly detection. As shown in Figure~\ref{fig:chnagescomb}, we show the different behaviors of the model as the system moves from the stable to unstable behavior. It shows presence of a transient phase before sticking to the unstable region. We present results which is calculated using equation~\eqref{eqn:discrepancy} in figure~\ref{fig:discrepancy}. It is interesting to see that the symbol emission probabilities are quite close to the conditionals during the stable operating condition, indicating symbolic noise behavior. While gradually there is an observed discrepancy between the conditional and marginal distribution of the symbol emissions that saturates as the system moves to the unstable behavior. Thus the current formalism is able to infer the underlying changes in the model structures as well as the parameters in an unsupervised fashion. This is encouraging as these changes can be associated with the changes in the physical dynamics during the complex process and thus, gives a high-fidelity statistical model for the process.
\begin{figure*}
	\centering
	\subfloat[Chnages in the complexity of the model inferred from data]{\includegraphics[width=0.45\textwidth]{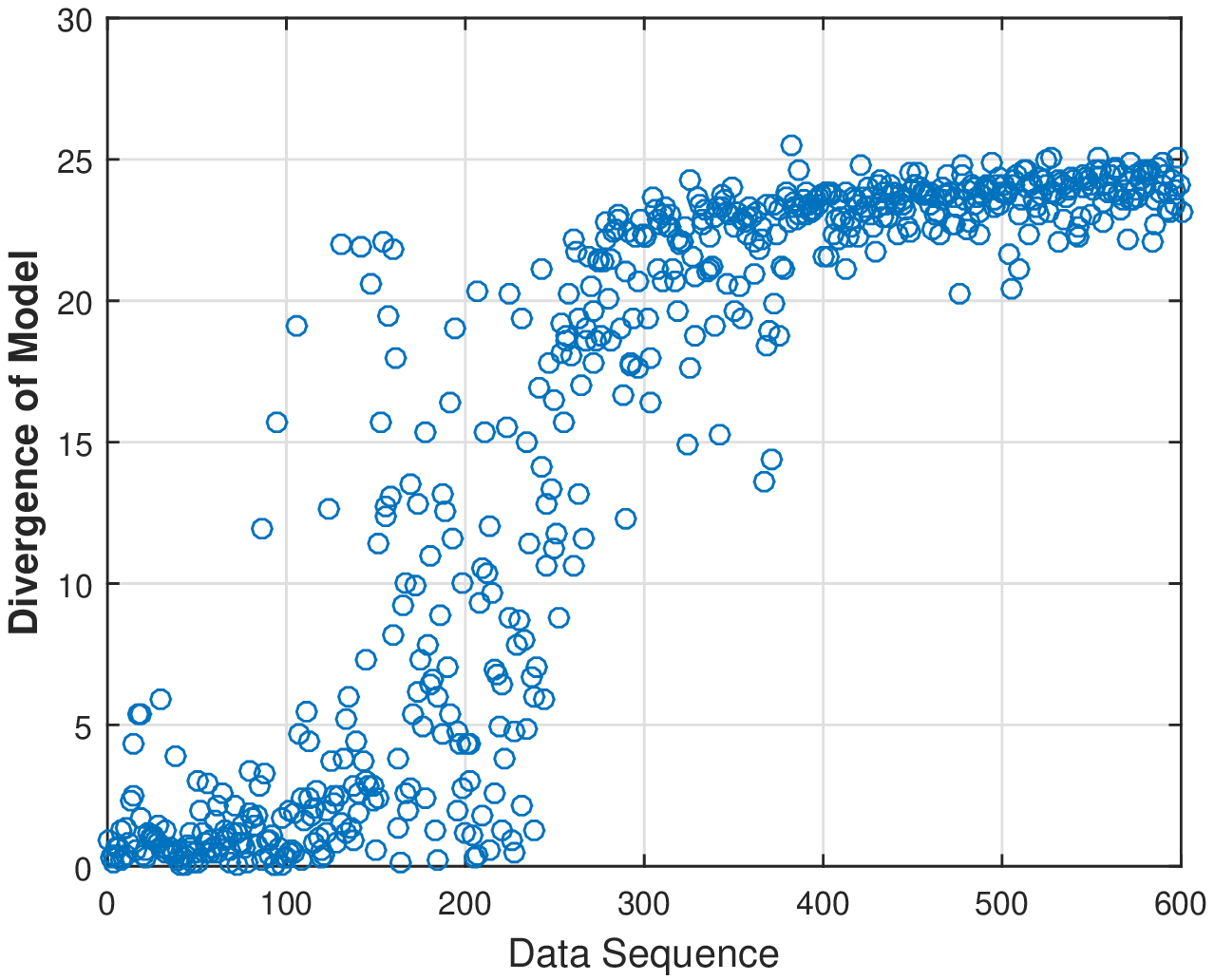}\label{fig:complexorig}}\quad
	\subfloat[The RMS of the pressure time series corresponding to the different models showing the range of stable and unstable behavior]{\includegraphics[width=0.45\textwidth]{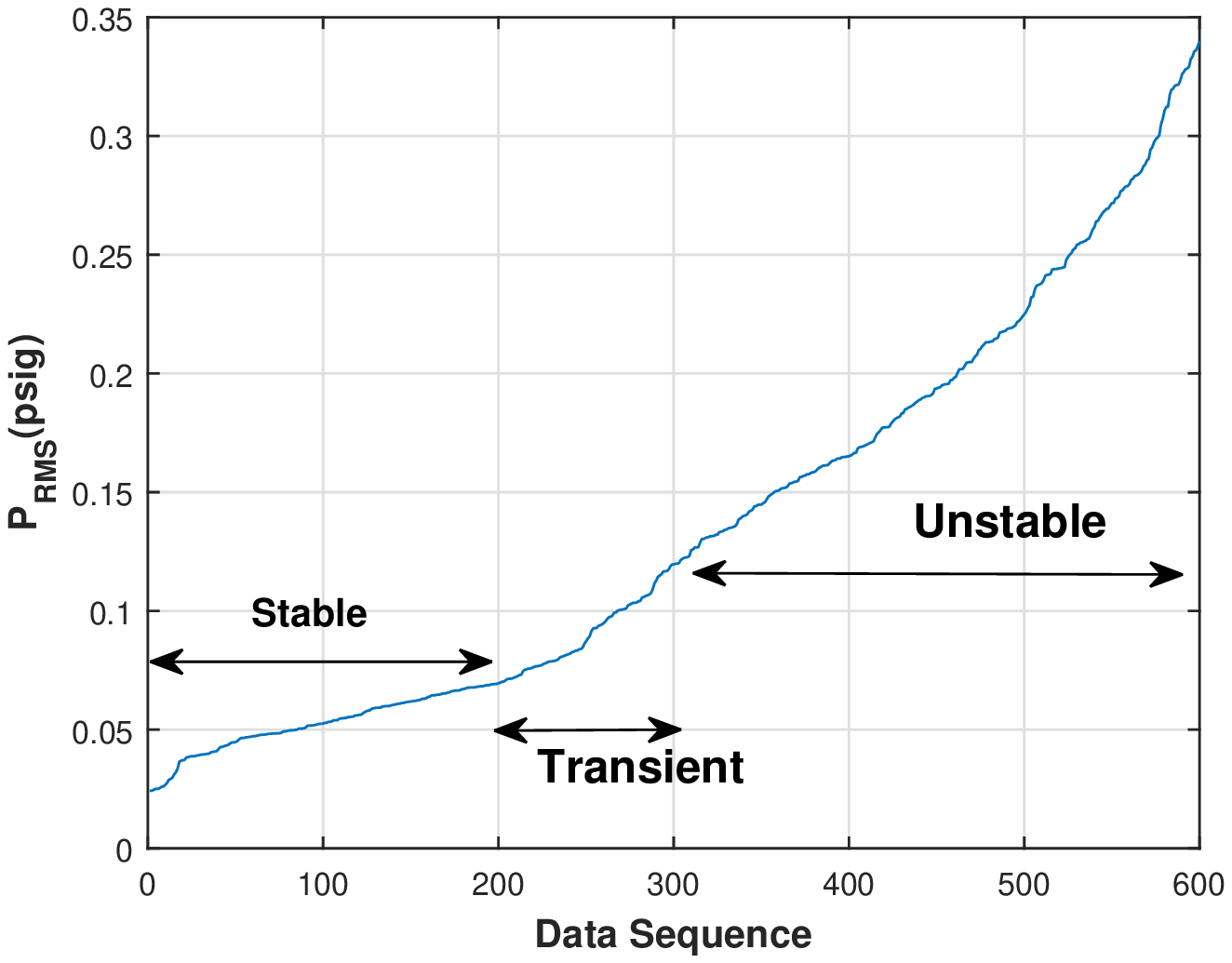}\label{fig:complexfinal}}\\
	\caption{Change in the Markov models of the pressure time-series data during the combustion process}
	\label{fig:chnagescomb}\vspace{-2pt}
\end{figure*}

\begin{figure*}[thb] %Fig02
	\centering \vspace{-6pt}
	\includegraphics[width=0.7\textwidth]{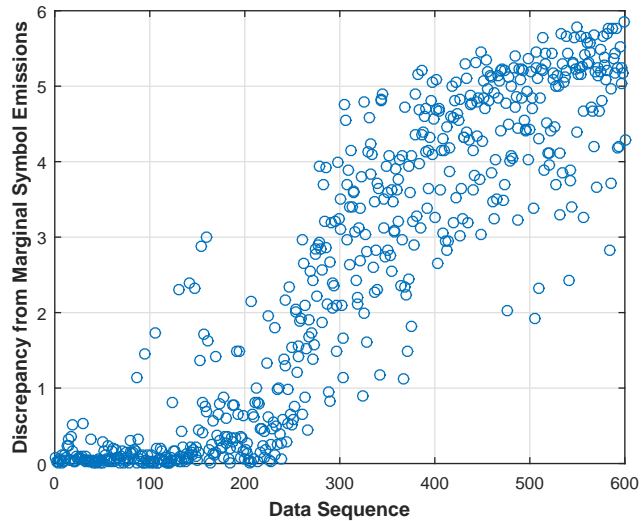}
	\caption{Change in the discrepancy of the symbol emission probabilities indicating the information gain by the Markov models}
	\label{fig:discrepancy}
\end{figure*}

The results presented in the above plots are inspiring as we see that the models corresponding to different behaviors are clustered together, irrespective of the other variables (like equivalence ratio, length of combustor, etc.). Thus the Markov models are able to capture features intrinsic to the governing physical dynamics independent of the other variables and thus serves as a good representation of the pressure signals obtained during the process.
\section{Conclusions and Future Work}\label{sec:Conclusions}
In this paper we presented the concept of Markov modeling of time-series data using Symbolic analysis. Symbolic analysis-based Markov modeling is a recently introduced statistical modeling approach where the data is first discretized and then, the discrete time-series data is approximated as a finite order Markov model. Compared to the widely-used Hidden Markov models, the present concept is algorithmically simple and thus could be easily inferred for machine learning applications. This is specially useful for embedded applications where using Hidden Markov model inference methods like Viterbi algorithm is computationally expensive and thus might be infeasible to use. 

We discussed that  the efficacy of the proposed approach depends on the characteristics of the two processes of discretization and Markov modeling of discrete processes. In this paper, we focused on the two critical steps in the modeling process namely, discretization and order estimation. We visited the key concepts of Markov partitioning from the dynamical systems literature and discussed some properties of the same. We also presented an order estimation technique for discrete symbolic process and showed the consistency of the estimation which guarantees almost sure convergence to the true order of the Markov chain. Use of current mathematical formalism for partitioning or discretization of data with unknown model is a very challenging problem. The results in dynamical systems theory is insufficient for use in machine learning applications. On the other hand, there is no rigorous statistical analysis of the discretization problem available in open literature. The order estimation problem for discrete stochastic processes is more rigorously studied in information theory literature. Consistency of some approaches have been rigorously established. However, for the statistical learning problems using STSA framework, the composite process of discretization and modeling of discrete process needs to be studied together. We also presented a case study of statistical learning for prognostics and anomaly detection in combustion process which is a complex thermoacoustic phenomena with undesirable effects. The Markov modeling approach is able to identify and represent changes in the signals of pressure time-series as the underlying physical process undergoes some change.

One possible direction for research is to find a discretization technique which results in order one discrete Markov process. Proving such a discretization always exists might be possible under stationarity conditions for systems with fading memory. For example, the metric introduced in equation~\eqref{eqn:discrepancy} provides a measure to describe the discrepancy between the independent statistics of the symbols versus the statistics when conditioned on the states of a probable Markov chain. Maximizing such a discrepancy provides a way to maximize the information gain obtained by $1^{st}$-order Markov model for the underlying data. However, further investigation is required to answer various related questions in this regard. We state the problem more formally next. \\
\textbf{Problem 1.} Let the time-series data be denoted as the sequence $\{X_t\}_{t \in \natno}$ where $X_t \in \Omega \subseteq \real^d$ and let $\varphi$ represent the partitioning function such that $\varphi: X_t \mapsto a_t$ where $a_t \in \alphabetSet$ for all $t \in \natno$ and $\mid\alphabetSet \mid \in \natno$ is known and fixed. Then the partitioning function is completely determined by the set $\mathcal{R}_\varphi =\{R_0,R_1,\dots, R_{|\alphabetSet|}\}$ such that $\Omega= \overline{R_0}\cup \overline{R_1}\cup\dots \cup \overline{R_{| \alphabetSet|}}$. Let us assume that the possible family of sets lies in a set denoted by $\Omega_\varphi$. An information gain for the Markov model could be measured by the following equation.
\begin{equation}\label{eqn:partitionIG}
d_{\mathcal{R}_\varphi} = \sum_{q\in Q} \prob(q) D_{KL}(P(\mathcal{A}\mid q)\| \tilde{P}({\mathcal{A}}))
\end{equation}
where the measure is parameterized by the set $\mathcal{R}_\varphi$ which depends on the partitioning function $\varphi$ and the set $\stateSet$ represents the finite memory-words of the discrete symbol sequence. The above equation measures the information gain by creating a Markov model, where we measure the symbol emission probabilities conditioned on the memory words (or states) of the Markov model, over the independent or marginal symbol emission probabilities. Another interpretation is that equation~\eqref{eqn:partitionIG} describes the discrepancy between the statistics of $\{s_t\}_{t \in \natno}$ when modeled as independent sequence versus when modeled as a stationary Markov process. Then, a partitioning to optimize this measure may capture the true temporal behavior of $\{X_t\}_{t \in \natno}$. The problem is to obtain the parameters of the related optimization problem.
\begin{equation}
\mathcal{R}_\varphi^\star=\arg\max\limits_{\mathcal{R}_\varphi \in \Omega_\varphi} \sum_{q\in Q} \prob(q) D_{KL}(P(\mathcal{A}\mid q)\| \tilde{P}({\mathcal{A}})) \nonumber
\end{equation}
The following questions need to be answered to characterize $\mathcal{R}_{\varphi}^\star$.
\begin{itemize}
\item Is $\mathcal{R}_{\varphi}^\star$ unique? Under what conditions of the underlying process $\{X\}_{t \in \natno}$, can we get an unique solution?
\item What is the order of the corresponding discrete time-series obtained for $\mathcal{R}_{\varphi}^\star$? 
\item Let us assume that the pre-image of a symbol is represented by the centroid of the set in $\mathcal{R}_\varphi^\star$ in the original phase-space. Then, how can we characterize the metric $\zeta_\varphi = \sum_{t\in \natno}d({\varphi}^{-1}(s_t),X_t)$ for signal representation? Is $\mathcal{R}_\varphi^\star$ able to minimize this metric over $\Omega_{\varphi}$?
\item Now imagine that the size of partitioning set is allowed to vary. Then, how to assign a MDL score to the individual models for different sizes of the partitioning set and how do we select a final model for signal representation? 
\end{itemize}

The above problem would, thus, try to formulate and characterize the properties of a partition for a data set for Markov representations. In the next question, we will try to study the composite problem of signal representation by partitioning followed by order estimation. \\
\textbf{Problem 2.} Let the time-series data be denoted as the sequence $\{X_t\}_{t \in \natno}$ where $X_t \in \Omega \subseteq \real^d$ and let $\varphi$ represent the partitioning function such that $\varphi: X_t \mapsto a_t$ where $a_t \in \alphabetSet$ for all $t \in \natno$ and $\mid\alphabetSet \mid \in \natno$ is unknown. A desirable way to characterize a partitioning process is by predicting its effect on the size of temporal memory of the system. Is it possible to synthesize a partitioning function $\varphi_I$ which preserves the memory of the discrete system under the transformation $\varphi_I$, i.e., if $\prob(X_t\mid X_{t-D},\dots , X_{t-1})=\prob(X_t\mid X_{t-1},\dots)$ then we have $\prob(\varphi_I(X_t)\mid \varphi_I(X_{t-D}),\dots ,\varphi_I( X_{t-1}))=\prob(\varphi_I(X_t)\mid \varphi_I(X_{t-1}),\dots)$. Then, how can we characterize a system for which such a discretization is guaranteed to exist? It is noted that the size of partitioning set is unknown and not fixed. 

These two problems could be treated as fundamental problems that need to be studied for mathematical characterization of data-driven modeling of systems from a symbolic analysis perspective. They are mainly concerned about inference of model structure for statistical learning. There are some problems which are of interest from applications perspective and related to estimation of various parameters during modeling and inference. However, they are not being presented here.
%Based on the discussions presented in this paper, future work will focus on further investigation of the discretization and order estimation process. In particular, future work will explore some techniques for Markov partitioning of time-series data. Furthermore, currently, these two are studied separately neglecting the effect of discretization on the order of the discrete sequence. 
%%
\clearpage
\bibliographystyle{ieeetr}
\bibliography{MathMAPaper}

\end{document}